\documentclass[nohyperref]{article}
\makeatletter
\def\@xfootnote[#1]{%
  \protected@xdef\@thefnmark{#1}%
  \@footnotemark\@footnotetext}
\makeatother
\usepackage{microtype}
\usepackage{graphicx}
\usepackage{subfig}  % added
\usepackage{booktabs} % for professional tables

\usepackage{hyperref}
\usepackage{url}            % simple URL typesetting

\newcommand{\repo}{\url{https://github.com/robertcsordas/linear_layer_as_attention}}
\usepackage[accepted]{icml2022}

\usepackage{amsmath}
\usepackage{amssymb}
\usepackage{mathtools}
\usepackage{amsthm}

\usepackage{xcolor}         % colors
\usepackage{color}

\usepackage{multirow}
\usepackage{amsmath,amsfonts,bm}

{}
{}
{}
{}

\def\eqref#1{equation~\ref{#1}}
\def\1{\bm{1}}

\def\va{{\bm{a}}}
\def\vb{{\bm{b}}}

\def\ve{{\bm{e}}}

\def\vk{{\bm{k}}}

\def\vq{{\bm{q}}}

\def\vv{{\bm{v}}}

\def\vx{{\bm{x}}}
\def\vy{{\bm{y}}}

\def\mA{{\bm{A}}}

\def\mE{{\bm{E}}}

\def\mK{{\bm{K}}}

\def\mV{{\bm{V}}}
\def\mW{{\bm{W}}}
\def\mX{{\bm{X}}}

\DeclareMathAlphabet{\mathsfit}{\encodingdefault}{\sfdefault}{m}{sl}
\SetMathAlphabet{\mathsfit}{bold}{\encodingdefault}{\sfdefault}{bx}{n}

\newcommand{\relu}{\mathrm{relu}}

\newcommand{\softmax}{\mathrm{softmax}}
\newcommand{\Attention}{\mathrm{Attention}}
\newcommand{\FFN}{\mathrm{FFN}}

\usepackage{graphicx}
\newtheorem{prop}{Proposition}
\newtheorem{col}{Corollary}

\newcommand{\bluet}[1]{\textbf{\textcolor{blue}{#1}}}

\usepackage[capitalize,noabbrev]{cleveref}

\theoremstyle{plain}
\newtheorem{theorem}{Theorem}[section]

\newtheorem{lemma}[theorem]{Lemma}

\theoremstyle{definition}
\newtheorem{definition}[theorem]{Definition}

\theoremstyle{remark}
\newtheorem{remark}[theorem]{Remark}

\usepackage[textsize=tiny]{todonotes}

\icmltitlerunning{The Dual Form of Neural Networks Revisited}

\begin{document}

\twocolumn[
\icmltitle{The Dual Form of Neural Networks Revisited: Connecting Test Time Predictions to Training Patterns via Spotlights of Attention}
\icmlsetsymbol{equal}{*}

\begin{icmlauthorlist}
\icmlauthor{Kazuki Irie$^*$}{idsia}
\icmlauthor{R\'obert Csord\'as$^*$}{idsia}
\icmlauthor{J\"urgen Schmidhuber}{idsia,kaust}
\end{icmlauthorlist}

\icmlaffiliation{idsia}{The Swiss AI Lab, IDSIA, USI \& SUPSI, Lugano, Switzerland}
\icmlaffiliation{kaust}{AI Initiative, King Abdullah University of Science and Technology (KAUST), Thuwal, Saudi Arabia}

\icmlcorrespondingauthor{}{\{kazuki, robert, juergen\}@idsia.ch}

% \begin{icmlauthorlist}
% \icmlauthor{Firstname1 Lastname1}{equal,yyy}
% \icmlauthor{Firstname2 Lastname2}{equal,yyy,comp}
% \icmlauthor{Firstname3 Lastname3}{comp}
% \end{icmlauthorlist}



\icmlkeywords{Machine Learning, ICML}

\vskip 0.3in
]

\printAffiliationsAndNotice{\icmlEqualContribution} % otherwise use the standard text.

\begin{abstract}
Linear layers in neural networks (NNs) trained by gradient descent can be expressed as a key-value memory system which stores all training datapoints and the initial weights, and produces outputs using unnormalised dot attention over the entire training experience. While this has been technically known since the 1960s, no prior work has effectively studied the operations of NNs in such a form, presumably due to prohibitive time and space complexities and impractical model sizes, all of them growing linearly with the number of training patterns which may get very large. However, this dual formulation offers a possibility of directly visualising how an NN makes use of training patterns at test time, by examining the corresponding attention weights. We conduct experiments on small scale supervised image classification tasks in single-task, multi-task, and continual learning settings, as well as language modelling, and discuss potentials and limits of this view for better understanding and interpreting how NNs exploit training patterns.
Our code is public\footnote[$\dagger$]{\repo}.
\end{abstract}

\section{Introduction}
\label{sec:intro}
Despite the broad success of neural nets (NNs) in many applications, 
much of their internal functioning
remains obscure.
Naive visualisation of their activations or weight matrices rarely shows human-interpretable patterns, with the occasional exception of certain special structures such as filters in convolutional NNs trained for image processing \citep{ZeilerF14}, attention weights \citep{bahdanau2014neural} 
or the sequential Jacobian \citep{Graves2008} in sequence processing,
or, to a limited extent, the distribution of individual word embeddings in natural language processing (NLP) \citep{MikolovSCCD13}.
In many ways,  NNs remain blackboxes.
In particular, while iteratively trained on a large number of datapoints, 
the entire training experience gets \textit{somehow} compressed (through lossy compression)
into a fixed size weight matrix,
which may be useful for making predictions on yet unseen datapoints,
although the raw values of weights are not easily human-interpretable.

While it is debatable whether this lack of interpretability is a problem,
it makes it hard for humans to  explain certain practical results obtained by NNs---especially those trained on a vast amount of data.
For example, how can big language models such as GPT-3 \citep{gpt3} generate an answer to a question never seen during training, translate languages without being trained to do so, or solve previously unseen math problems?
How can DALL-E \citep{RameshPGGVRCS21} generate various pictures of a \textit{radish walking a dog} or a \textit{banana performing stand-up comedy} without training examples containing such images\footnote{We assume that this was the case.}?
Perhaps there have been at least some pictures of ``radish''
and others representing the concept of ``walking a dog'' among the training samples,
and somehow the model \textit{interpolated} them\footnote{
Here by ``interpolate'', we informally mean ``combine''.
For a discussion based on a formal definition of ``interpolation'', see e.g., \citet{balestriero2021learning}.
}.
If so, is it possible to point out exactly which training samples are the \textit{original sources} of that specific output?

Here we propose to revisit the \textit{dual form of the perceptron} \citep{aizerman1964theoretical} and apply it in the modern context of deep NNs, with the objective of better understanding how training datapoints relate to test time predictions in NNs.
Essentially, the dual form expresses the forward operation of any linear layers in NNs trained by gradient descent (GD) as a key/value/query-attention operation \citep{trafo} where the keys and values are training datapoints and the query is generated from the test input (details in Sec.~\ref{sec:main}). This allows for directly connecting test time predictions to training datapoints.
To the best of our knowledge, no prior work has studied deep NNs in their dual form,
which is not surprising, considering certain obvious computational drawbacks (see Sec.~\ref{sec:main}).

Importantly, none of the mathematical results we'll discuss is novel: originally introduced by \citet{aizerman1964theoretical}, the presentation of the perceptron \citep{rosenblatt1958} in its primal and dual forms is well-known
and often repeated in the literature and in textbooks \citep{scholkopf2002learning, bishop2006PRML},
especially in the context of support vector machines \citep{BoserGV92, burges1998tutorial},
for the case where the output of a linear layer is one-dimensional (i.e., its weight matrix reduces to a vector).
Unlike prior works, however, our focus is on the dual form of linear layer operations based on weights in matrix form (arguably the most frequently used operations in NNs) inside a deep NN trained by gradient descent.
Also, while there is no technical gain in expressing the dual form in terms of key-value/attention concepts \citep{trafo, MillerFDKBW16, SukhbaatarSWF15, graves2014neural},
such a formulation has become very popular
since the advent of Transformers \citep{trafo}, and we believe many modern readers will find our equations straightforward to interpret.\looseness=-1

As a first empirical work analysing deep NNs under this view,
our main experiments are based on small scale models and datasets in image classification and language modelling.
For image classification,
we analyse feedforward NNs with two hidden layers using the MNIST \citep{lecun1998mnist} and Fashion-MNIST \citep{xiao2017fashion} datasets.
We start with the single task scenario to illustrate our basic approach.
Then we investigate multi-task and continual learning scenarios.
Finally, we conduct a similar analysis in the NLP domain. We train language models based on long short-term memory recurrent NNs \citep{hochreiter1997long} on a small public domain book and the WikiText-2 dataset \citep{merity2016pointer}.

\section{Preliminaries}
\label{sec:prel}

The results presented here are preliminary in the sense that they are
prerequisites needed to derive and express the core results of  Sec.~\ref{sec:main}.
While all proofs are trivial, and some of the results are well-known, we still opt for presenting them in detail, in the form  {\em Definition/Lemma/Proposition} (and {\em Corollary} later in Sec.~\ref{sec:main}) to concisely highlight the core concepts and their relations in a self-contained manner.
Please refer to our Related Work Section \ref{sec:rel} for further comments and references.

We first introduce the two following definitions used throughout this paper.
In what follows, let $d_{\text{in}}$, $d_{\text{out}}$ and $T$ denote positive integers.

\begin{definition}[Unnormalised Dot Attention]
\label{def:attn}
Let $\mK = (\vk_1, ..., \vk_T) \in \mathbb{R}^{d_{\text{in}} \times T}$ and $\mV = (\vv_1, ..., \vv_T) \in \mathbb{R}^{d_{\text{out}} \times T}$
denote matrices representing $T$ key and value vectors.
Let $\vq \in \mathbb{R}^{d_{\text{in}}}$ denote a query vector. An unnormalised linear dot attention operation $\Attention(\mK, \mV, \vq)$ (``attention'' for short)
computes the following weighted average of value vectors $\vv_t$:
\begin{eqnarray}
\Attention(\mK, \mV, \vq) = \sum_{t=1}^T \alpha_t \vv_t
\end{eqnarray}
where the weights $\alpha_t = \vk_t^\intercal \vq \in \mathbb{R}$ are dot products between key $\vk_t$ and query $\vq$ vectors, and are called \textit{attention weights}.\\
\end{definition}

\begin{lemma}[] 
\label{lemma1}
$\Attention$ computation defined above can be expressed as:
\begin{eqnarray}
\label{eq:attn_mat}
\Attention(\mK, \mV, \vq) = \mV\mK^\intercal\vq
\end{eqnarray}
\end{lemma}

\begin{proof}
\begin{align}
\Attention(\mK, \mV, \vq) &= \sum_{t=1}^T \alpha_t \vv_t = \sum_{t=1}^T \vv_t \alpha_t \\
 &= \sum_{t=1}^T \vv_t \vk_t^\intercal \vq = \big(\sum_{t=1}^T \vv_t \vk_t^\intercal \big) \vq
\end{align}
we obtain Eq.~\ref{eq:attn_mat} since the term in the parentheses is:
\begin{align}
\label{eq:outer}
\sum_{t=1}^T \vv_t \vk_t^\intercal = \sum_{t=1}^T \vv_t \otimes \vk_t = \mV\mK^\intercal
\end{align}
where $\otimes$ denotes the outer product.
\end{proof}

\begin{remark}
Obviously, referring to the equations above as \textit{attention} (or unnormalised attention)
emphasizes the relation to the regular softmax normalised
dot product attention \citep{luong2015, bahdanau2014neural} which is $\mV\softmax(\mK^\intercal\vq)$ using the notations above.
We also note that  \citet{schmidhuber1993reducing} referred to the unnormalised attention used in the context of fast weight controllers as ``internal spotlights of attention.''
\end{remark}

\begin{definition}[Equivalent Systems, for shortcut]
Two systems $S_1$ and $S_2$ defined over the same input and output domains $\mathcal{D}_{\text{in}}$ and $\mathcal{D}_{\text{out}}$,
are said to be \textit{equivalent} if and only if for any input, their outputs are equal, i.e.,
for any $\vx \in \mathcal{D}_{\text{in}}$, the following holds
\begin{eqnarray}
S_1(\vx) = S_2(\vx)
\end{eqnarray}
\end{definition}
This allows us to informally talk about ``equivalence'' between two models regardless of their computational complexities (which might differ).

The following proposition expresses a general statement on the duality between the linear transformation with a \textit{fixed size} weight matrix constructed as a sum of outer-products between two known vectors
and an \textit{arbitrary size} attention-based key-value memory.
We apply this result later in the case of linear layers in an NN trained by gradient descent.

\begin{figure}[t]
    \begin{center}
        \includegraphics[width=.4\columnwidth]{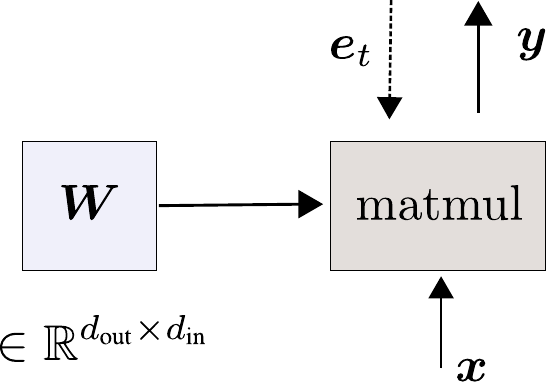}
        \caption{The primal form of a linear layer to be contrasted with the dual form in Figure \ref{fig:dual}.}
        \label{fig:primal}
    \end{center}
\end{figure}

\begin{prop}[Attention/Linear Layer Duality]
\label{prop1}
Let  $\mK = (\vk_1, ..., \vk_T) \in \mathbb{R}^{d_{\text{in}} \times T}$ and $\mV = (\vv_1, ..., \vv_T) \in \mathbb{R}^{d_{\text{out}} \times T}$ denote matrices with $T$ column vectors.
The following two systems $S_1$ and $S_2$ are equivalent:

$S_1$ (Linear layer): A system consisting of one weight matrix $\mW \in \mathbb{R}^{d_{\text{out}} \times d_{\text{in}}}$ constructed as:
\begin{eqnarray}
\mW = \sum_{t=1}^{T} \vv_t  \otimes \vk_t
\end{eqnarray}
which transforms input $\vx \in \mathbb{R}^{d_\text{in}}$ to output $S_1(\vx) \in \mathbb{R}^{d_\text{out}}$ as:
\begin{eqnarray}
S_1(\vx) = \mW \vx
\end{eqnarray}
$S_2$ (Attention layer): A system consisting of memory storing $T$ key-value pairs $(\vk_1, \vv_1), .., (\vk_T, \vv_T)$ which transforms input $\vx \in \mathbb{R}^{d_\text{in}}$ to output $S_2(\vx) \in \mathbb{R}^{d_\text{out}}$ as:
\begin{eqnarray}
S_2(\vx) = \Attention(\mK, \mV, \vx)
\end{eqnarray}
\end{prop}
\begin{proof}
We have almost shown this already in Eq.~\ref{eq:outer}.
Given $\vx \in \mathbb{R}^{d_\text{in}}$, $\mK = (\vk_1, ..., \vk_T) \in \mathbb{R}^{d_{\text{in}} \times T}$ and $\mV = (\vv_1, ..., \vv_T) \in \mathbb{R}^{d_{\text{out}} \times T}$,
by starting from the form of attention shown in Lemma \ref{lemma1}, we obtain
\begin{align}
S_2(\vx) & = \Attention(\mK, \mV, \vx) = \mV\mK^\intercal\vx \\
    & = \big(\sum_{t=1}^{T} \vv_t  \otimes \vk_t \big) \vx = \mW \vx = S_1(\vx) \qedhere
\end{align}
\end{proof}
This is essentially a general formulation of calculation used to show the equivalence between linear models and kernel machines in the 1960s (\citet{aizerman1964theoretical}; cf.~Related Work Sec.~\ref{sec:rel}).
We express it for an arbitrary weight matrix constructed as a sum of outer-products, and using the modern language of key-value/attention.

\begin{figure}[t]
    \begin{center}
        \includegraphics[width=1.\columnwidth]{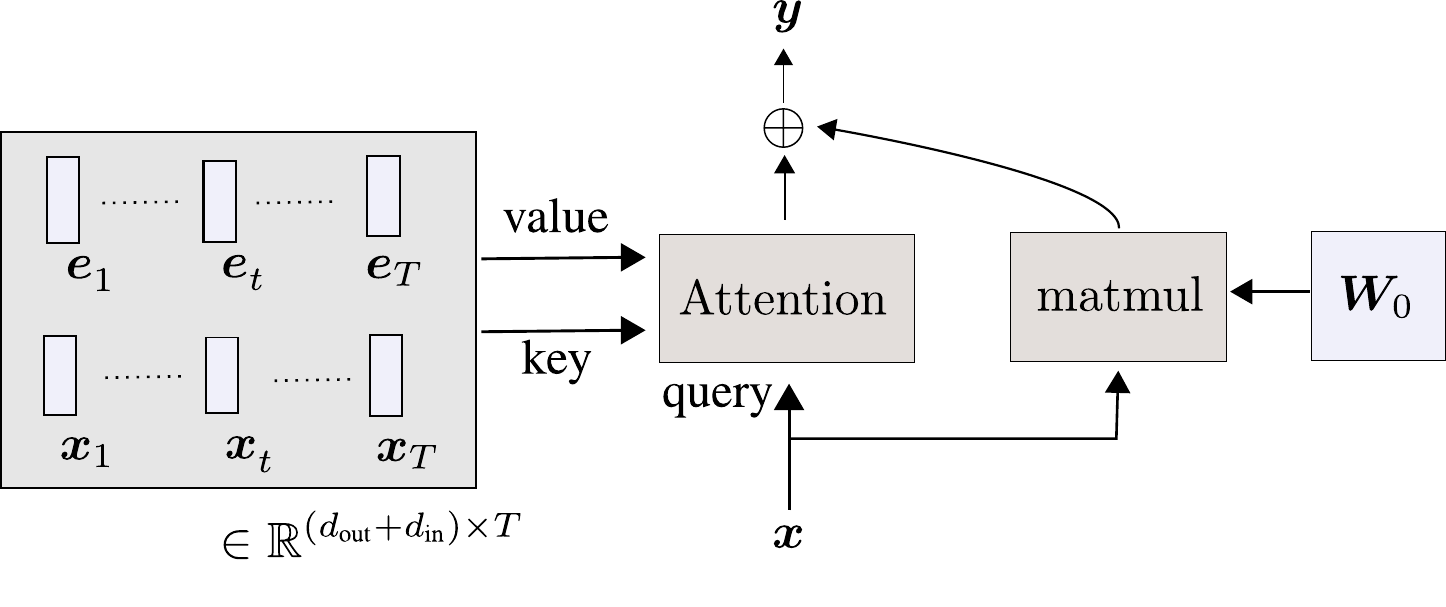}
        \caption{The dual form of a linear layer trained by gradient descent is a key-value memory with attention storing the entire training experience. Compare to the primal form in Figure \ref{fig:primal}.}
        \label{fig:dual}
    \end{center}
\end{figure}

\section{The Dual Form of Linear Layers in NNs Trained by Gradient Descent}
\label{sec:main}

The following corollary obtained from the proposition above is the core result explored
in the experimental section.
It expresses the dual form of a linear layer
trained by gradient descent as a key-value system storing training patterns as key-value pairs,
which computes the output from a test query using attention over the key-value memory.

This duality is illustrated in Figures \ref{fig:primal} and \ref{fig:dual}.

\begin{col}[Dual Form of a Linear Layer Trained by GD]
\label{col}
The following two systems $S_1$ and $S_2$ are equivalent:

$S_1$ (Primal form): A linear layer in a neural network trained by gradient descent in some error function using $T$ training inputs to this layer $(\vx_1, ..., \vx_T)$ with $\vx_t \in \mathbb{R}^{d_{\text{in}}}$ and corresponding (backpropagation) error signals\footnote{
In the case of standard gradient descent using a loss $\mathcal{L}$,
$\ve_t = -\eta_t (\nabla_{\vy} \mathcal{L})_t$
where $\eta_t \in \mathbb{R}$ is the learning rate
and $\vy_t = \mW_t \vx_t$ is the output of the linear layer using the weight matrix $\mW_t$ at step $t$.
}
$(\ve_1, ..., \ve_T)$ with $\ve_t \in \mathbb{R}^{d_{\text{out}}}$ obtained by gradient descent. Its weight matrix $\mW \in \mathbb{R}^{d_{\text{out}} \times d_{\text{in}}}$ is thus:
\begin{eqnarray}
\mW = \mW_0 + \sum_{t=1}^{T} \ve_t \otimes \vx_t
\end{eqnarray}

where $\mW_0 \in \mathbb{R}^{d_{\text{out}} \times d_{\text{in}}}$ is the initialisation.
The layer transforms input $\vx \in \mathbb{R}^{d_\text{in}}$ to output $S_1(\vx) \in \mathbb{R}^{d_\text{out}}$ as:
\begin{eqnarray}
S_1(\vx) = \mW \vx
\end{eqnarray}

$S_2$ (Dual form): A layer which stores $T$ key-value pairs $(\vx_1, \ve_1), .., (\vx_T, \ve_T)$
i.e.,
a key matrix $\mX = (\vx_1, ..., \vx_T) \in \mathbb{R}^{d_{\text{in}} \times T}$
and a value matrix
$\mE = (\ve_1, ..., \ve_T) \in \mathbb{R}^{d_{\text{out}} \times T}$, and a weight matrix $\mW_0 \in \mathbb{R}^{d_{\text{out}} \times d_{\text{in}}}$
which transforms input $\vx \in \mathbb{R}^{d_\text{in}}$ to output $S_2(\vx) \in \mathbb{R}^{d_\text{out}}$ as:
\begin{eqnarray}
\label{eq:dual_linear}
S_2(\vx) = \mW_0 \vx + \Attention(\mX, \mE, \vx)
\end{eqnarray}
\end{col}

\begin{proof}
The result is trivially obtained by applying Proposition \ref{prop1} to $\Attention(\mX, \mE, \vx)$ in Eq.~\ref{eq:dual_linear}.
\end{proof}

We note that unlike the primal form, the computational complexity of the dual form $S_2$ above depends on the
number of training datapoints $T$ the layer is trained on (i.e., the number of all inputs forwarded to the linear layer during training for which the loss is computed; that is, in a typical mini-batch stochastic gradient descent setting, the batch size times the number of training iterations).
The time and space complexities of the $\Attention$ computation in Eq.~\ref{eq:dual_linear}
is linear in $T$, while $T$ can be potentially very large.
The model size is especially prohibitive in practical scenarios:
the parameters of the dual form are the ($d_{\text{out}} \times d_{\text{in}}$)-dimensional initial weight matrix $\mW_0$
---whose size already equals the number of parameters in the primal form---
and the $((d_{\text{in}} + d_{\text{out}}) \times T)$-sized
key-value memory recording all training datapoints.
The dual form is thus not a practical form for regular settings.
Indeed, this duality of Proposition \ref{prop1} has been used in the converse direction to obtain time and space efficient attention computation (see Sec.~\ref{sec:rel}/Linear Transformers).

However, there are also benefits in viewing NNs under the dual form.
It explicitly shows that the output of an NN linear layer trained by backpropagation is mainly a linear combination of the training error signals $\ve_t$ the layer receives during training:
$\sum_1^T \alpha_t \ve_t$, where the weights $\alpha_t$
are computed by comparing the test query to each training input.
This is potentially more interpretable than the primal form as the attention weights $\alpha_t = x_t^\intercal x$ should indicate which training datapoints are ``activated'' for a given test input.
The main contribution of this paper is to
effectively implement this dual form
and visualise the corresponding attention
in different scenarios in Sec.~\ref{sec:exp}.
Eq.~\ref{eq:dual_linear} reveals further notable properties
listed as remarks as follows:

\begin{remark}[\textbf{Nothing is ``forgotten''}]
In a system expressed as in the dual form $S_2$ of Corollary \ref{col}, \textit{nothing is forgotten}.
The entire life of an NN is recorded and stored as a key matrix $\mX$ and value matrix $\mE$.
Roughly speaking, the only limitation of the model's capability to ``remember'' something
is the limitation of the retrieval process (we illustrate this in the continual learning experiments in Sec.~\ref{sec:continual}).
Also, note that the $\Attention$ computation in Eq.~\ref{eq:dual_linear}
remains invariant if we shuffle the order of columns in
the key-value storage,
as it is carried out without any explicit positional encoding (similarly to auto-regressive attention \citep{irie19:trafolm, tsai2019} in some Transformer language models).
The time information is naturally encoded into the key and value vectors (except for the key vector in the first layer) tracking the recurrence of the training process in time.
\end{remark}

\begin{remark}[\textbf{Unlimited memory size is not necessarily useful}]
This duality also illustrates an important fact (which might seem counter-intuitive at first glance) that systems which store everything in memory by increasing its size for each new event ($S_2$) are not necessarily better than those with a fixed size storage ($S_1$).
The retrieval mechanism has to be powerful enough to exploit the stored memory.
Potentially, we might obtain better models
by using other more powerful kernels (e.g., softmax, like in the standard Transformers)
in the $\Attention$ computation in Eq.~\ref{eq:dual_linear}.
That would, however, require to use the dual form even during training, which is prohibitive.
\end{remark}

\begin{remark}[\textbf{Orthogonal Inputs}]
If an input $\vx$ was orthogonal to all training input patterns, the output of the linear layer would be $\mW_0 \vx$ (as $\Attention(\mX, \mE, \vx)=0$ in Eq.~\ref{eq:dual_linear}),
i.e., the weights learned by gradient descent would not contribute to the output.
\end{remark}

\begin{remark}[\textbf{Non-Uniqueness}]
The expression of a linear layer as an attention system is not unique.
Some tensor product decompositions can be applied to a trained weight matrix to obtain a more compact attention system.
A clear benefit of the one presented in Corollary \ref{col}, however, is that it explicitly relates test inputs to training datapoints.
\end{remark}

\begin{remark}[\textbf{Self-Attention as Two Level Retrieval}]
Since this duality is valid for any linear layer trained by gradient descent,
and such a linear layer is ubiquitous in any NN,
viewing a well known NN architecture from this view could give extra insights and interpretation.
For example, the main transformation in a common self-attention \citep{trafo} is a linear projection layer which transforms the input to key/value/query vectors.
Under the dual form, this can be viewed as a hierarchical retrieval layer where the projection layer conducts a first level of retrieval using unnormalised dot attention on the training datapoints.
The second, regular softmax attention conducts a second level of retrieval among those selected by the first level.
\end{remark}

\section{Related Work}
\label{sec:rel}

\paragraph{Relating Kernel Machines and NNs.}
As stated in Sec.~\ref{sec:intro}, the theoretical results we show in Sec.~\ref{sec:prel} and \ref{sec:main} are
special cases or trivial extensions of the lines of works connecting kernel machines and neural networks derived and presented in different contexts.
Most importantly, it is well known that the perceptron \citep{rosenblatt1958} has primal and dual forms as pointed out by \citet{aizerman1964theoretical}.
The corresponding result is often presented in the case 
where the co-domain of the linear layer is one-dimensional (i.e., the weight matrix reduces to a weight vector),
especially in the literature on support vector and kernel machines \citep{BoserGV92, burges1998tutorial} and in textbooks \citep{scholkopf2002learning, bishop2006PRML}.
In Sec.~\ref{sec:main}, we present the result in the case of linear layers with matrix weights and expressed it in the form of key-value memory which is particularly relevant today \citep{trafo}.
But we note that the same statement could be made using the language of kernel machines as the unnormalised dot attention operation can be expressed as:
\begin{eqnarray}
\label{eq:kernel}
\Attention(\mX, \mA, \vx) + \vb = \sum_{t=1}^{T} \va_t K(\vx_t, \vx) + \vb
\end{eqnarray}
where $K$ denotes the dot product, $\va_t \in \mathbb{R}^{\text{out}}$ and $\vb \in \mathbb{R}^{\text{out}}$ are model parameters.

More recent work of \citet{domingos2020every} presents
a more general statement: an entire multi-layer perceptron can be approximated by a kernel machine.
While this is a powerful theoretical result, practical
ways of exploiting it are yet to be investigated.
We study instead the dual form of each linear layer in the network which can be obtained without any approximation.

\paragraph{Transformer Feedforward Block as Key-Value Memory.}
Another related but different line of works studies the   feedforward
block of Transformers as a key-value memory.
This feedforward block consists of two feedforward layers which transform input vector $\vx \in \mathbb{R}^{d_{\text{in}}}$ as follows:
\begin{eqnarray}
\FFN(\vx)=  \mW_2 \relu(\mW_1\vx)
\end{eqnarray}
with weight matrices $\mW_1 \in \mathbb{R}^{d_{\text{ff}} \times d_{\text{in}}}$ and $\mW_2 \in \mathbb{R}^{d_{\text{in}} \times d_{\text{ff}}}$,
where $d_{\text{ff}}$ denotes the inner dimension (we omit the biases).
The possibility to interpret this block as a key-value memory with attention
has been pointed out by the original authors of Transformers \citep{trafo}\footnote{See Appendix ``Two feedforward Layers = Attention over Parameter'' in version arXiv:1706.03762v3 of \citet{trafo}.}.
Essentially, replacing relu by a softmax,
\begin{eqnarray}
\mW_2 \softmax(\mW_1\vx)
\end{eqnarray}
or in our context by removing relu, we can write it down using unnormalised attention
defined by Definition \ref{def:attn} as:
\begin{eqnarray}
\mW_2 \mW_1\vx = \Attention(\mW_1^\intercal, \mW_2, \vx)
\end{eqnarray}
Using this formulation, \citet{sukhbaatar2019augmenting} proposed to merge the feedforward block and the attention layer by extending the context-dependent key/value vectors
in the regular self-attention with a fixed set of trainable key/value vectors.
More recently, \citet{gevaetal2021transformer} asked
the question what information from the training data these key vectors (i.e., rows of $\mW_1$) contain.
They compare the corresponding keys to activation vectors of training examples, which are obtained by computing the forward pass of the already trained model on the training examples to be analysed.
They conduct such analyses for Transformer language models.

The view we explore here is different.
Our statement is not limited to Transformer feedforward blocks, but to any linear layers trained by gradient descent,
and  we directly express a linear layer as a function of the training datapoints.

\paragraph{Attention and Kernels.}
A number of recent works connect
attention and kernels \citep{tsai2019, katharopoulos2020transformers, choromanski2020rethinking, peng2021random}.
While we also implicitly exploit the corresponding connection (Eq.~\ref{eq:kernel}),
our focus is on connecting the primal form, i.e., linear layers trained by gradient descent, to key-value/attention systems.

\paragraph{Fast Weight Programmers and Linear Transformers.}
As we noted while discussing the complexity
of the dual form in Sec.~\ref{sec:main}, the conversion from the primal to the dual form of a linear layer we explore here (Proposition \ref{prop1}) is
used in prior works connecting Transformers with linearised attention and fast weight programmers \citep{ba2016using, katharopoulos2020transformers, schlag2021linear}.
\citet{ba2016using} show the equivalence of unnormalised attention to fast weight programmers of the '90s \citep{schmidhuber1993reducing}.
\citet{katharopoulos2020transformers} convert the self-attention in Transformers---whose complexity is quadratic in sequence length---to a linear-complexity linear layer with fast weights \citep{Schmidhuber:91fastweights}.

%%%%%%%%%%%%%%%%%%%%%%%%%%%%%%%%%%%%%%%%

\section{Experiments}
\label{sec:exp}
We posit that the dual formulation of linear layers (Corollary \ref{col}) offers
a possibility to visualise how different training datapoints (i.e., a set of input/error signal pairs seen during training) contribute 
to some NN's test time predictions.
Here we experimentally support this claim.
We effectively study NNs in their dual form and conduct analyses based on the attention weights and related metrics described below.
We consider three different scenarios: single-task training, multi-task \textit{joint} training, and multi-task \textit{continual} training
for image classification using feedforward NNs with two hidden layers.
We also conduct experiments with language modelling using a one-layer LSTM recurrent NN.

\subsection{Common Settings}
We first describe settings which are common to all our experiments.
Since our primary interest is to analyse attention weights of Eq.~\ref{eq:dual_linear} which are simply dot products between
a test \textit{query} $\vx$ and each of the training \textit{keys} $\vx_t$,
the general idea is to train an NN on a given dataset by backpropagation
as usual,
while recording all inputs to each linear layer during training\footnote{For analyses requiring information about the error signals $\ve_t$
(e.g.,~their norm), we also need to record those vectors.}.
After training, attention weights for the trained model for any test example for all training datapoints can be computed by forwarding the test example to the model, to obtain the input vector (test query) to each linear layer, then, we compute the corresponding dot products with the stored training inputs (training keys).
We note that the training procedure does not change.

Since storing the inputs to each linear layer for the entire training experience can be highly demanding in terms of disk space,
we work with small datasets:
MNIST \citep{lecun1998mnist} and Fashion-MNIST \citep{xiao2017fashion}
for image classification\footnote{We also considered CIFAR-10 \citep{krizhevsky} but omitted it as we did not achieve accuracies above 50\% using small feedforward NNs.}, WikiText-2 \citep{merity2016pointer} and a small public domain book for language modelling.
This allows for obtaining well-performing models with small size within a reasonably small number of training steps.

Our model for image classification has two hidden layers with 800 nodes, each using relu activation functions after each layer.
This means that the model has three linear layers which we denote as \textit{layer-0}, \textit{layer-1}, and \textit{layer-2}.
The first layer-0 transforms a gray-scale input image of size 768 (28$\times$28) to a 800-dimensional hidden state, layer-1 transforms it to another 800-dimensional hidden state, and finally layer-2 projects it to a 10-dimensional output.
The total storage of training patterns is thus roughly 3$\times$800$\times T$ units,
where $T$ is the ``number of training datapoints'' counting  all examples across all training mini-batches.
We train this model using the vanilla stochastic gradient descent optimiser.
We specify relevant metrics in the respective sections below.
Many more examples with visualisation are shown in Appendix \ref{app:examples}.

\begin{figure}[h]
\hspace{8mm}
	\subfloat[MNIST class 6]{
		\centering
		\includegraphics[width=.3\linewidth]{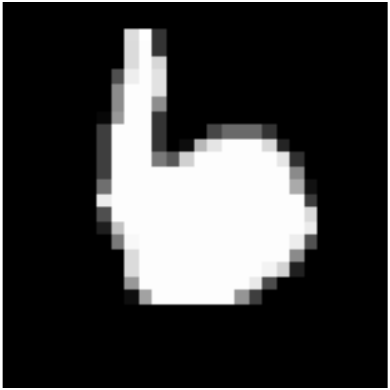}
		\label{sfig:input_mnist}
	}
\hspace{10mm}
	\subfloat[F-MNIST class 9]{
		\centering
		\includegraphics[width=.3\linewidth]{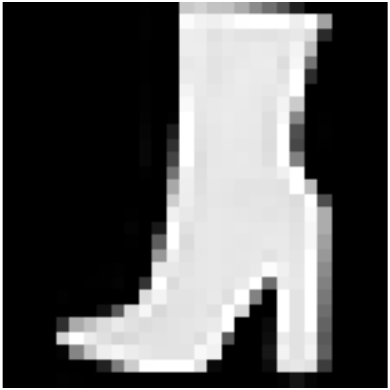}
		\label{sfig:input_f_mnist}
	}
	\caption{\it Test examples used in Figures \ref{fig:mnist_heatmaps} \& \ref{fig:mnist_per_class_sum} (single task case) and Figures \ref{fig:f_mnist_heatmaps} \& \ref{fig:f_mnist_per_class_sum} (multi-task joint training case), respectively.}
	\label{fig:input_image}
		\vspace{-3mm}
\end{figure}

%%%%%%%% Single case
\begin{figure*}[ht]
	\subfloat[layer-0]{
		\centering
		\includegraphics[width=.32\linewidth]{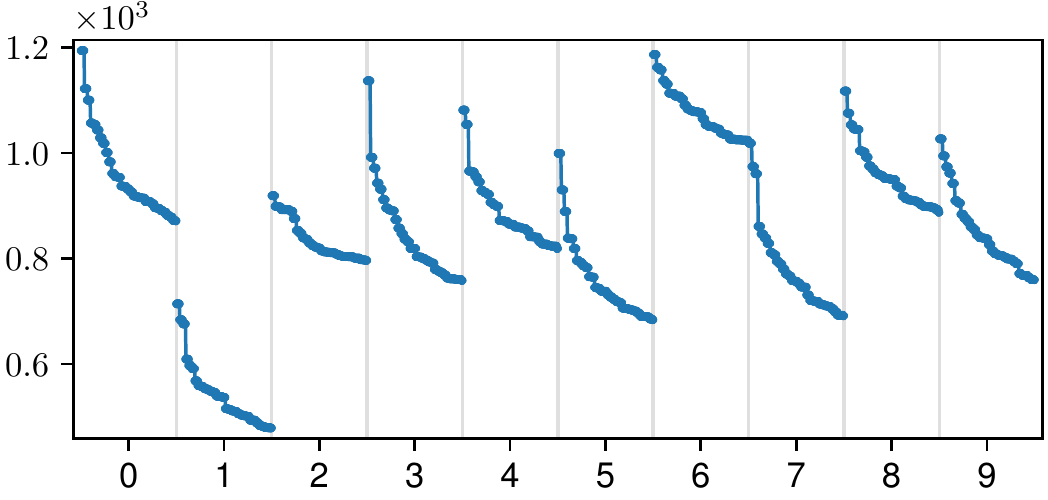}
		\label{sfig:mnist_layer0}
	}
	\subfloat[layer-1] {
		\centering
		\includegraphics[width=.32\linewidth]{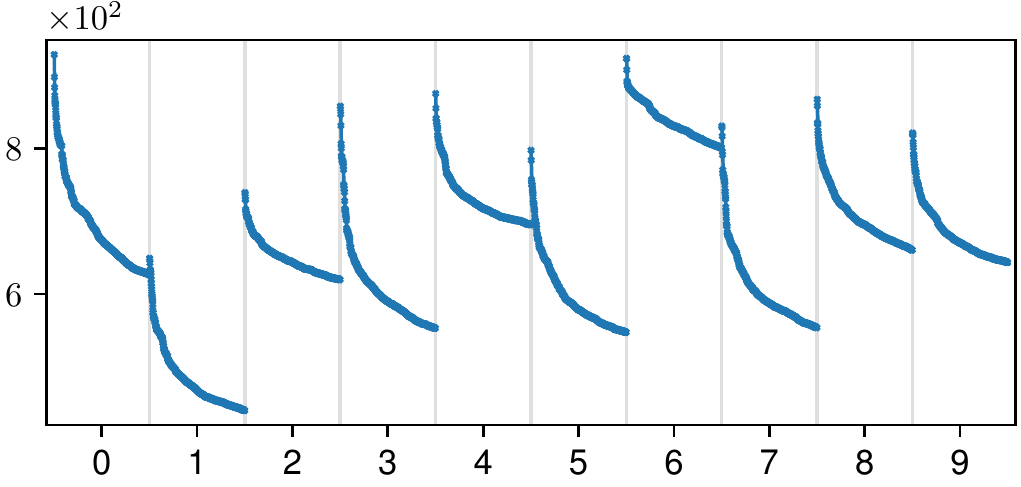}
		\label{sfig:mnist_layer1}
	}
	\subfloat[layer-2]{
		\centering
		\includegraphics[width=.32\linewidth]{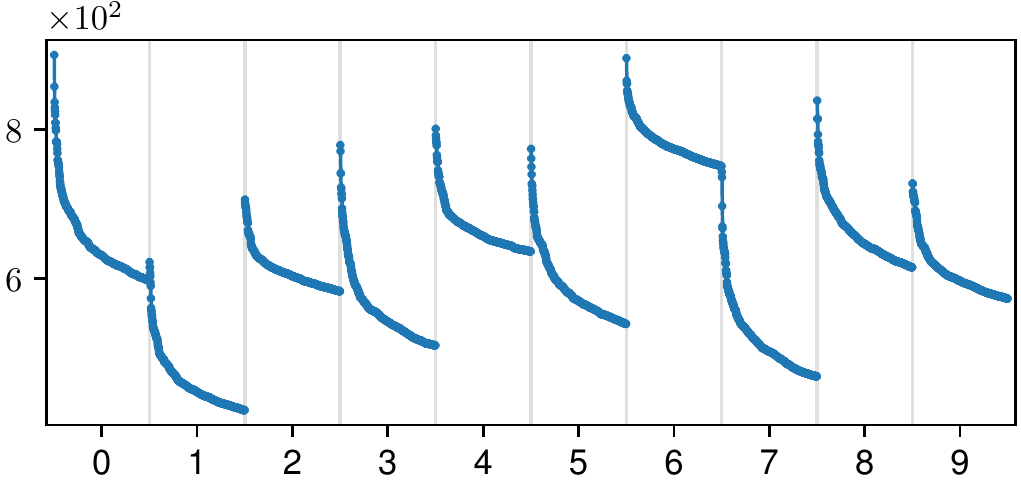}
		\label{sfig:mnist_layer2}
	}
	%\vspace{-3mm}
	\caption{\it Attention weights over training examples for the input test example from class 6 (Figure \ref{sfig:input_mnist}) for the single task case on MNIST. The x-axis is partitioned by class, and for each class, top-500 datapoints sorted in descending order are shown.}
	\label{fig:mnist_heatmaps}
		\vspace{-3mm}
\end{figure*}

\begin{figure*}[ht]
	\subfloat[layer-0]{
		\centering
		\includegraphics[width=.32\linewidth]{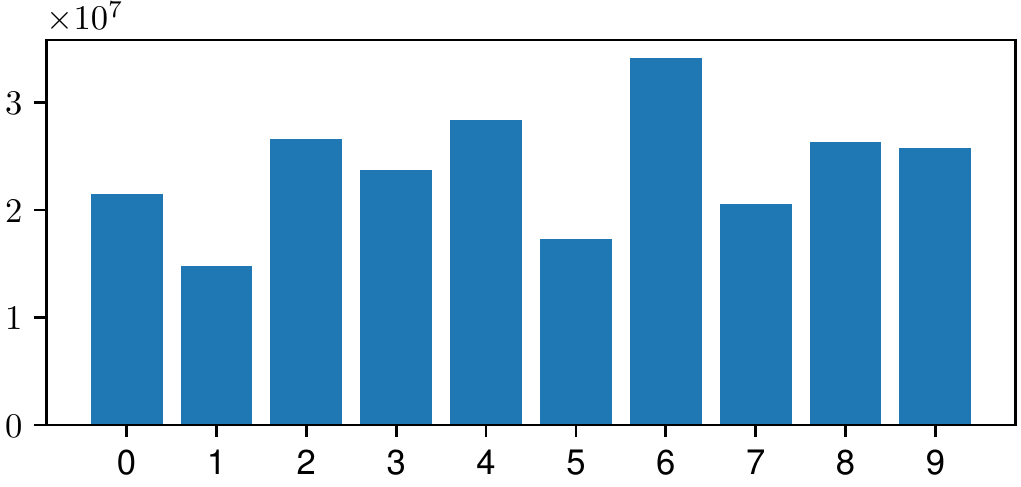}
		\label{sfig:mnist_layer0_sum}
	}
	\subfloat[layer-1] {
		\centering
		\includegraphics[width=.32\linewidth]{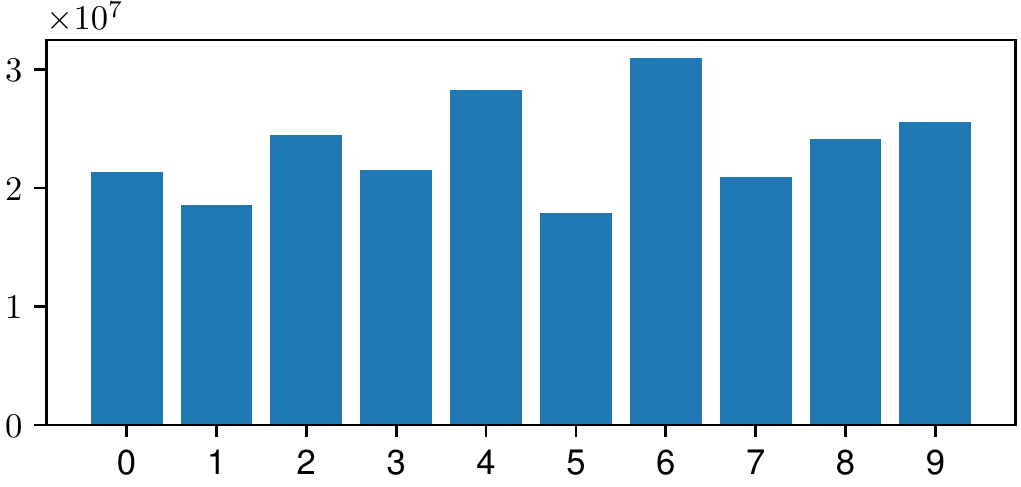}
		\label{sfig:mnist_layer1_sum}
	}
	\subfloat[layer-2]{
		\centering
		\includegraphics[width=.32\linewidth]{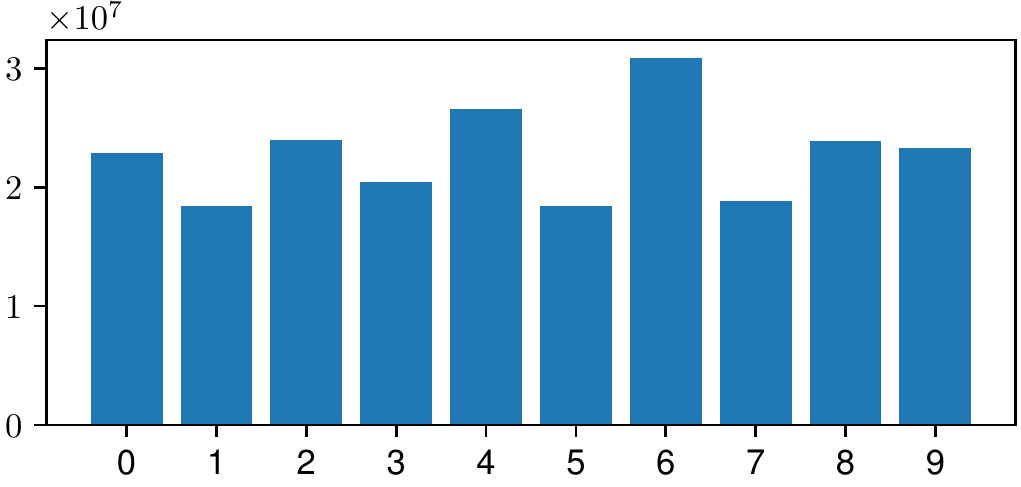}
		\label{sfig:mnist_layer2_sum}
	}
	%\vspace{-3mm}
	\caption{\it Total scores per class 
	for the input test example from class 6 (Figure \ref{sfig:input_mnist})
	for the single task case on MNIST.}
	\label{fig:mnist_per_class_sum}
		\vspace{-3mm}
\end{figure*}

%%%%%%%% Joint case
\begin{figure*}[ht]
	\subfloat[layer-0]{
		\centering
		\includegraphics[width=.32\linewidth]{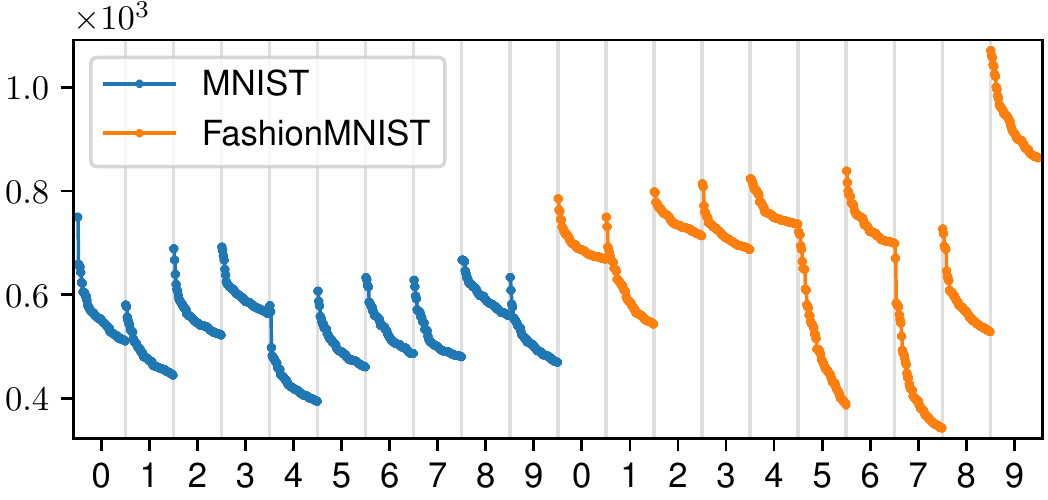}
		\label{sfig:f_mnist_layer0}
	}
	\subfloat[layer-1] {
		\centering
		\includegraphics[width=.32\linewidth]{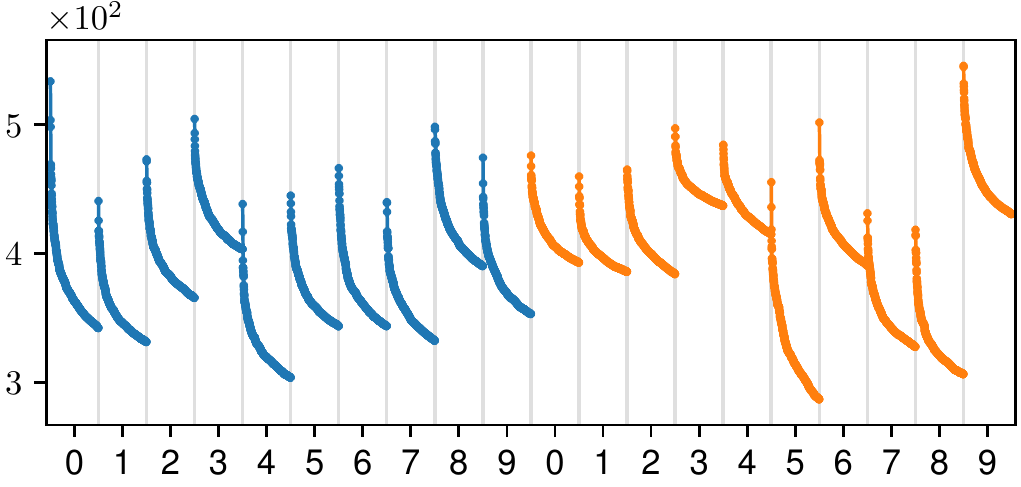}
		\label{sfig:f_mnist_layer1}
	}
	\subfloat[layer-2]{
		\centering
		\includegraphics[width=.32\linewidth]{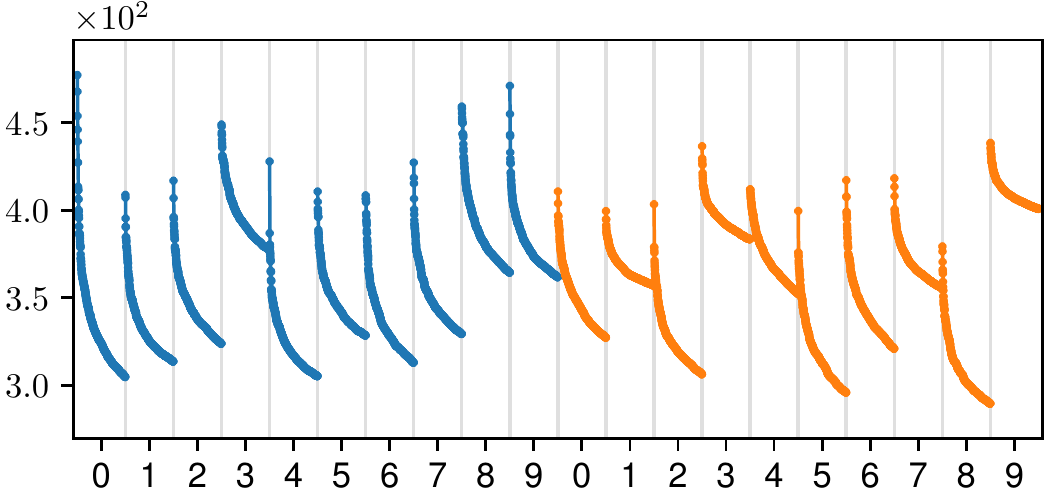}
		\label{sfig:f_mnist_layer2}
	}
	%\vspace{-3mm}
	\caption{\it Attention weights over training examples for the input test example from class 9 of F-MNIST (Figure \ref{sfig:input_f_mnist}) in the \textbf{joint training} case. The x-axis is partitioned by class (for each task), and for each class, top-500 datapoints sorted in descending order are shown.}
	\label{fig:f_mnist_heatmaps}
		\vspace{-3mm}
\end{figure*}

\begin{figure*}[ht]
	\subfloat[layer-0]{
		\centering
		\includegraphics[width=.32\linewidth]{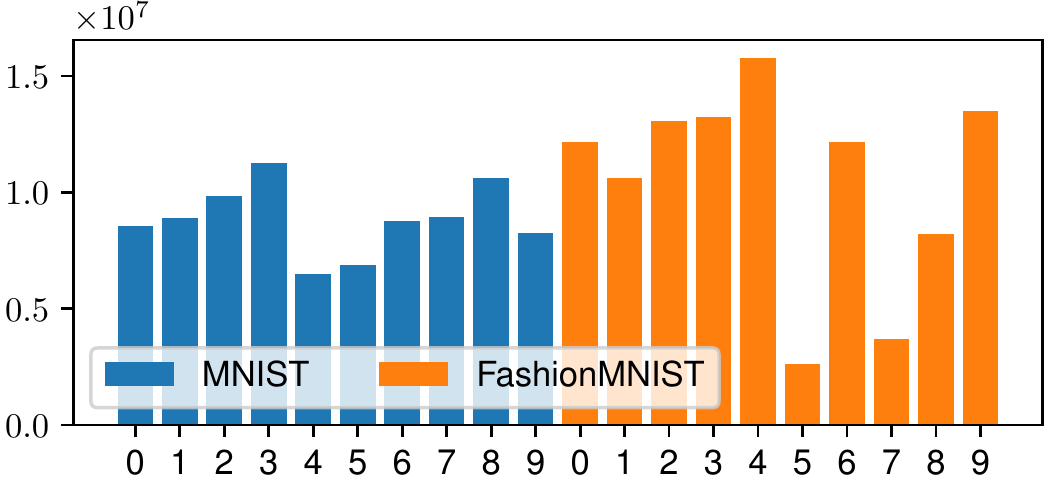}
		\label{sfig:f_mnist_layer0_sum}
	}
	\subfloat[layer-1] {
		\centering
		\includegraphics[width=.32\linewidth]{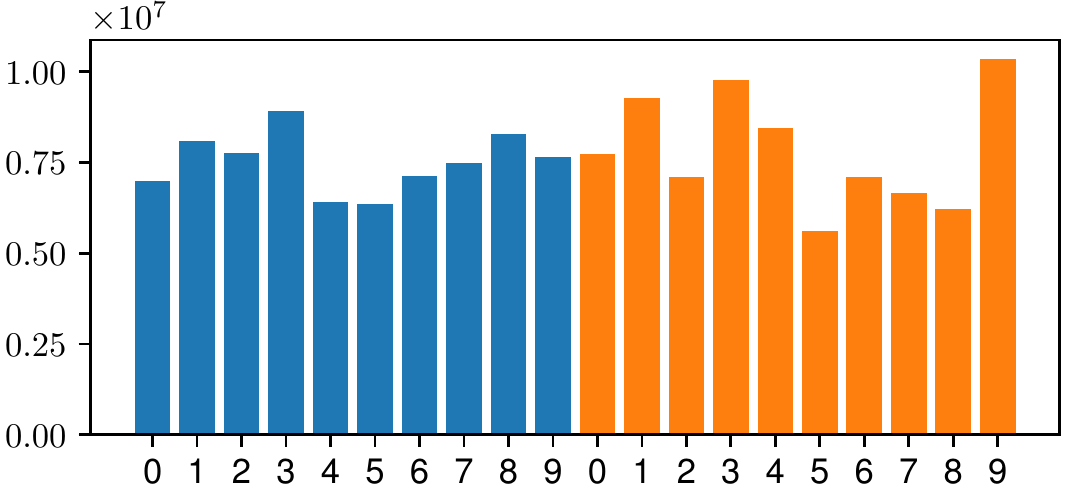}
		\label{sfig:f_mnist_layer1_sum}
	}
	\subfloat[layer-2]{
		\centering
		\includegraphics[width=.32\linewidth]{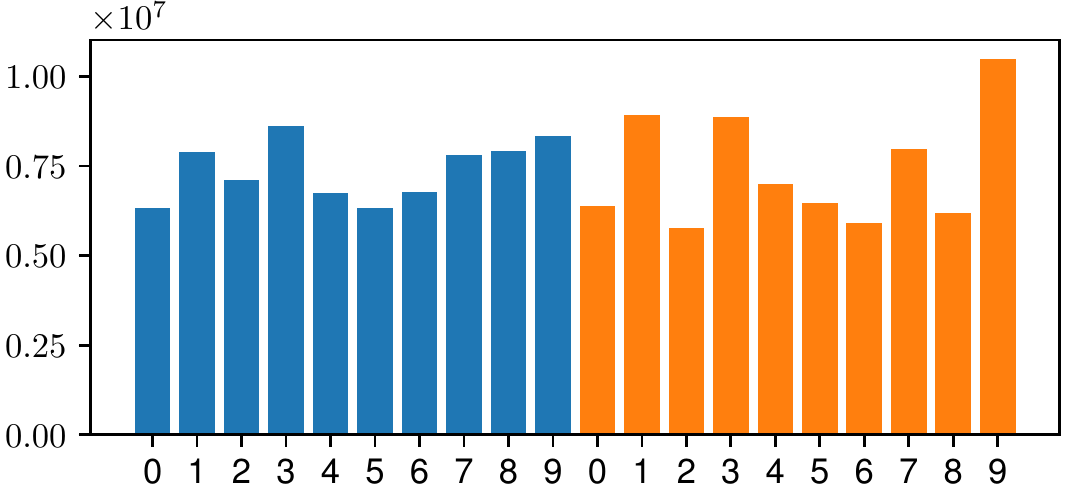}
		\label{sfig:f_mnist_layer2_sum}
	}
	%\vspace{-3mm}
	\caption{\it Total scores per class 
	for the input test example from class 9 of F-MNIST (Figure \ref{sfig:input_f_mnist})
	in the \textbf{joint training} case.}
	\label{fig:f_mnist_per_class_sum}
		\vspace{-3mm}
\end{figure*}

\subsection{Single Task Case}
\label{sec:single}

We start with the simple setting of image classification
where we train the two-hidden layer feedforward NN described above on the MNIST dataset \citep{lecun1998mnist}.
This first analysis also serves as the base case for
introducing the general idea and key metrics of our study.
The model is trained for 3\,K updates using a batch size of 128 which correspond to 384\,K-long training key/value memory slots.
The resulting model achieves 97\% accuracy on the test set.

The first plots we visualise are the \textit{attention weight ``distribution''} over training datapoints shown in Figures \ref{sfig:mnist_layer0}-\ref{sfig:mnist_layer2}.
The test sample fed to the model in this example is from class 6 (which is digit ``6''), shown in Figure \ref{sfig:input_mnist}, which is correctly classified by the model.
In these plots, datapoints are grouped by target class, and for each class, datapoints achieving the top-500 highest scores are shown in descending order (thus, a total of 5\,K datapoints out of 384\,K are shown) for each layer\footnote{We note that in layer-0, which is the input layer, the scores are simply dot products between the flattened raw image data vectors.}.
While the scores are unnormalised, these plots yield a qualitative picture of the attention ``distributions'' over classes on the level of datapoints.

However, this picture showing only the top-few-percent datapoints does not capture how much of the attention weights go to which class \textit{in total}.
In fact, in many cases, we notice that the training datapoints which get the highest scores are not necessarily from the correct class
(as these top scores may quickly decrease with the rank),
while the model's output prediction is correct.
For instance, in the example shown in Figures \ref{sfig:mnist_layer0}-\ref{sfig:mnist_layer2},
the input is an image from class 6 of MNIST (Figure \ref{sfig:input_mnist}).
We indeed see that training examples corresponding to class 6 are attended with high weights in all layers, but we also see that some of the datapoints for class 0 achieve comparable or higher scores than those from the correct class 6.
The corresponding top scoring examples are visualised in Figures \ref{sfig:mnist_layer0_top1}-\ref{sfig:mnist_layer2_top2} in Appendix \ref{app:top} which contain images from class 0.
In order to visualise the \textit{total} attention weights assigned to each class,
we present Figures \ref{sfig:mnist_layer0_sum}-\ref{sfig:mnist_layer2_sum} which show the sum\footnote{Here the sum is the sum of absolute values of attention scores for the input layer,
for which the scores might be negative.
For other layers, relu ensures positivity.} of attention scores per class.
Here we observe that training datapoints from the correct class label 6 are the most attended datapoints overall.
By examining multiple cases, we observe that the sum of attention weights effectively correlates well with the model output.
To quantify this observation, we compute the corresponding correlation on the entire test set.
Table \ref{tab:correlation} shows the results.
We observe that when the model's prediction is correct,
the total attention weights correlate well with the model output, especially in the late layers.

\begin{table}[h]
\caption{Prediction accuracy (\%) of the true target label (Target) and model output (Output) from argmax of the per-class attention scores (like those in Figure \ref{fig:mnist_per_class_sum}),
shown for each layer for two cases whether the model's output prediction is correct.
Mean and standard deviation computed for 5 runs.
Layer-0 is the input layer.
}
\label{tab:correlation}
% \vskip 0.15in
\begin{center}
\begin{tabular}{llrr}
\toprule
 &  & \multicolumn{2}{c}{Is Model Prediction Correct?} \\  \cmidrule(r){3-4}
Layer &  & \multicolumn{1}{c}{No} & \multicolumn{1}{c}{Yes} \\ \midrule
\multirow{2}{*}{0} & Target  & 17.2 $\pm$ 2.7 & \multirow{2}{*}{75.1 $\pm$ 0.0} \\
     & Output  & 49.5 $\pm$ 1.5 &  \\ \midrule
\multirow{2}{*}{1} & Target  & 18.0 $\pm$ 2.9 & \multirow{2}{*}{78.8 $\pm$ 0.8} \\
     & Output  & 52.9 $\pm$ 2.9 &  \\ \midrule
\multirow{2}{*}{2} & Target  & 20.6 $\pm$ 2.6 & \multirow{2}{*}{84.7 $\pm$ 1.1} \\
     & Output  & 60.1 $\pm$ 4.5 &  \\
\bottomrule
\end{tabular}
\end{center}
\end{table}

As a side note, we also considered an alternative version where we augment the attention weights with the norm of the error vector (value vector in Eq.~\ref{eq:dual_linear}).
However, the resulting plots did not show any consistent trend (presumably because the norm disregards the signs of each component of error/value vectors).
In fact, in general, the heatmap of attention weights (e.g.,~\citet{trafo, bahdanau2014neural}) often reported in the literature also completely disregards any information about the value vectors.
Our analysis is thus also purely based on attention weights, i.e., interaction between a test query and training key vectors.

Regarding the time information in layer-1 and layer-2 (it
does not matter in the input layer-0, since it is simply a dot product between the raw images 
which contains no such information),
we find the distribution over time to be rather spread out and noisy among the top-500 highest scoring datapoints.
We do not observe a clear pattern indicating that the most recent training datapoints are more important.

\subsection{Multi-Task Case}
\label{sec:joint}
Now we extend the experimental setting above by training models with the same architecture as above, but jointly on two datasets: MNIST and Fashion-MNIST (F-MINST for short).
The output dimension remains 10.
The model is trained for 5,000 steps and achieves 97\% and 87\% test accuracy on MNIST and F-MNIST, respectively.

In case of joint training, a natural question
to ask is how the attention behaves with two datasets and how much cross-task attention occurs (i.e.,~high attention scores when key and query are from different datasets).
Intuitively, since the output representation is shared (labels between 0 and 9),
representations learned by the late layers might be
task-independent, and cross-task attention might take place.
We experimentally observe that this is indeed the case.
Figures \ref{sfig:f_mnist_layer0}-\ref{sfig:f_mnist_layer2} show the attention weights over training datapoints in the case of joint training, analogous to Figures \ref{sfig:mnist_layer0}-\ref{sfig:mnist_layer2} for the single task case.
The x-axis now shows 20 groups representing 10 classes from MNIST and 10 from F-MNIST.
Figure \ref{sfig:input_f_mnist} shows the test example fed to the model which is from class 9 of F-MINST (``ankle boot'').
On Figures \ref{sfig:f_mnist_layer0}-\ref{sfig:f_mnist_layer2}, we observe that, while in layer-0,
top matching training datapoint are dominated by F-MNIST datapoints from the same class, in layer-1 and layer-2, top attention weights are more distributed across two tasks.
Especially in layer-2, the top-3 matching
examples (shown in Figures \ref{sfig:f_mnist_layer0_top1}-\ref{sfig:f_mnist_layer2_top2} in Appendix \ref{app:top}) contain datapoints from MNIST, one of them belonging to class 9.
Analogous to the single task case, Figures \ref{sfig:f_mnist_layer0_sum}-\ref{sfig:f_mnist_layer2_sum} show the sum of weights for each class (separately for each dataset).
Another interesting observation: while the class achieving the highest total attention scores in layer-0 is class 4 (``coat'') of F-MNIST (computed as a dot product between the raw images),
in later layers, the score of the correct label 9 is increased,
which is not surprising if we expect the late layers to yield better representations of inputs.
This is a visual illustration of the results we saw in Table \ref{tab:correlation} where the argmax of the sum of attention weights correlates better with the model output in the late layers.

\begin{figure*}[ht]
	\subfloat[layer-0]{
		\centering
		\includegraphics[width=.32\linewidth]{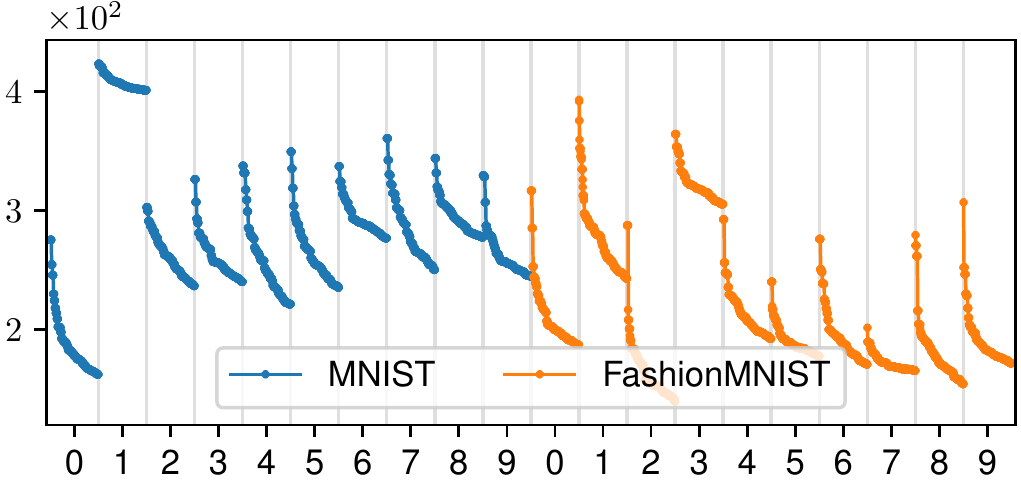}
		\label{sfig:mis_mnist_layer0}
	}
	\subfloat[layer-1] {
		\centering
		\includegraphics[width=.32\linewidth]{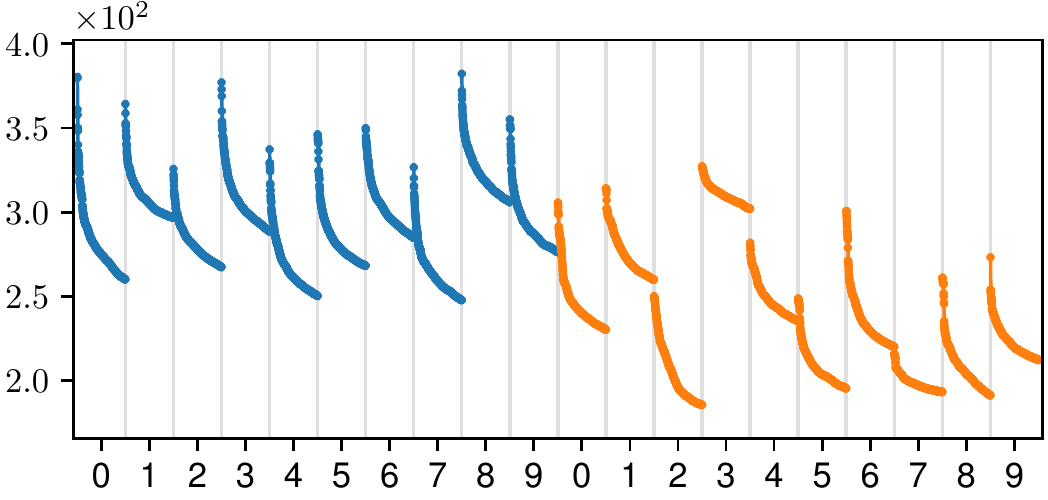}
		\label{sfig:mis_mnist_layer1}
	}
	\subfloat[layer-2]{
		\centering
		\includegraphics[width=.32\linewidth]{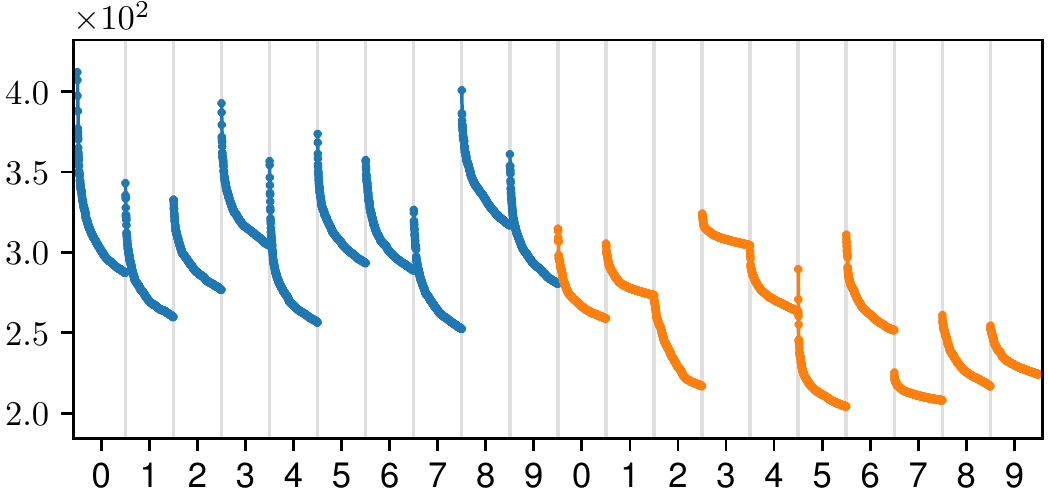}
		\label{sfig:mis_mnist_layer2}
	}
	\caption{\it Plots analogous to Figure \ref{fig:f_mnist_heatmaps} for the input test example from class 1 of MNIST (Figure \ref{sfig:app_input_seq_misclassified}) in the \textbf{continual learning case}.}
 	\label{fig:mis_mnist_heatmaps}
		\vspace{-3mm}
\end{figure*}

\begin{figure*}[ht]
	\subfloat[layer-0]{
		\centering
		\includegraphics[width=.32\linewidth]{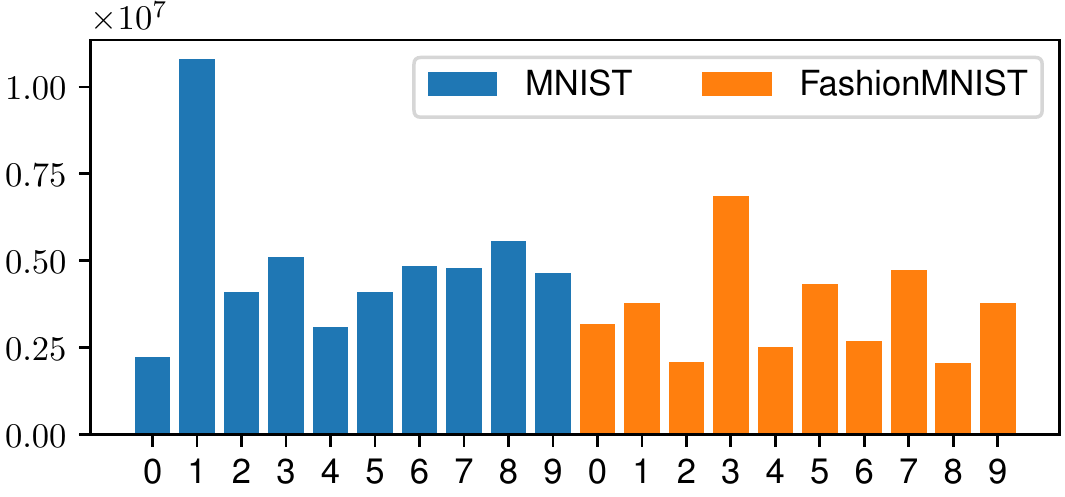}
		\label{sfig:mis_mnist_layer0_sum}
	}
	\subfloat[layer-1] {
		\centering
		\includegraphics[width=.32\linewidth]{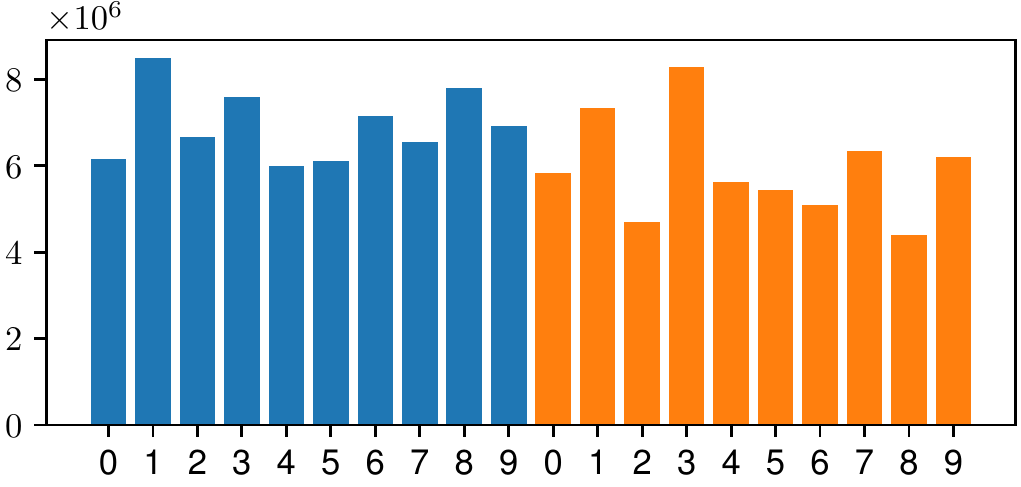}
		\label{sfig:mis_mnist_layer1_sum}
	}
	\subfloat[layer-2]{
		\centering
		\includegraphics[width=.32\linewidth]{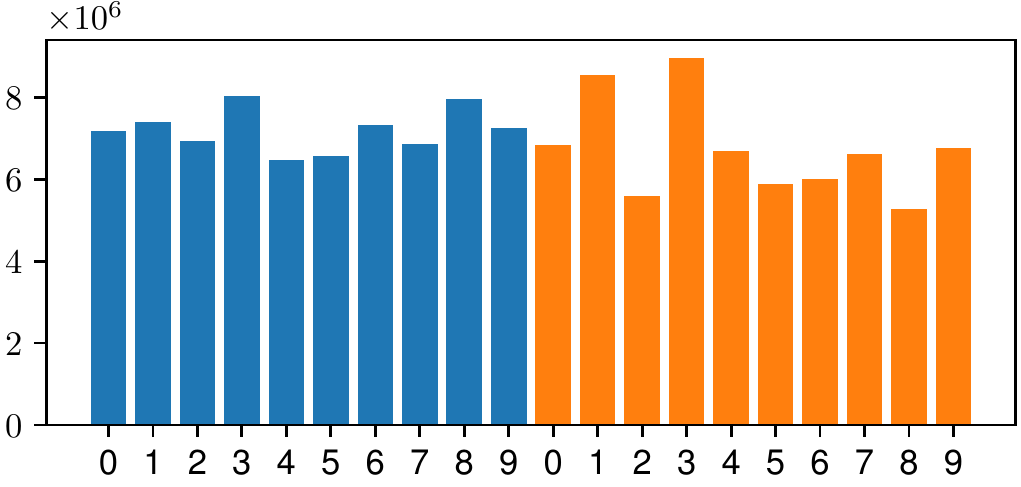}
		\label{sfig:mis_mnist_layer2_sum}
	}
	%\vspace{-3mm}
	% Index 1 
	\caption{\it Total scores per class
	for the input test example from class 1 of MNIST (Figure \ref{sfig:app_input_seq_misclassified})
	in the \textbf{continual training case}.}
 	\label{fig:mis_mnist_sum}
		\vspace{-3mm}
\end{figure*}

\subsection{Continual Learning Case}
\label{sec:continual}
Finally we train models on the two datasets successively
in a continual learning fashion, first  on MNIST for 3,000 steps, then on F-MINST for 3,000 steps.
The final test accuracies are 85\% on F-MNIST and 45\% on MNIST (down from 97\% after training on MNIST only).

The continual learning case is particularly interesting from the perspective of the dual form.
Since all training patterns are stored in the training key/value memory (Eq.~\ref{eq:dual_linear}), 
if we used an attention mask on the F-MNIST part of training keys for the attention computation, the model would be able to reproduce the good performance on MNIST obtained after the first part of training, and thus avoid degradation or more generally catastrophic forgetting \citep{french1999catastrophic}.
But there is no such a mask in practice: using an MNIST test input,
we observe strong attention weights on the F-MNIST part of training patterns,
causing severe interference.

In Figures \ref{fig:mis_mnist_heatmaps} and \ref{fig:mis_mnist_sum}, we present an illustrative example which shows how an MNIST test image clearly representing the digit ``1'' (shown in Figure \ref{sfig:app_input_seq_misclassified} in Appendix \ref{app:examples}) is wrongly recognised as ``3'' through interference of F-MNIST class 3 (see an example in Figure \ref{sfig:app_input_seq_fmnist}) which is ``dress''.
In the attention maps in Figure \ref{sfig:mis_mnist_layer0_sum}, we observe that the
MNIST class 1 training examples are highly activated, so are the class 3 of F-MNIST.
As we move up to layer-1 and layer-2 (Figures \ref{sfig:mis_mnist_layer1_sum} and \ref{sfig:mis_mnist_layer2_sum}), the attention becomes more distributed across two tasks, and we finally observe that in layer-2 (Figure \ref{sfig:mis_mnist_layer2_sum}), the class 3 of F-MNIST becomes the most dominant class,
which may explain the model's decision to output class 3.
Otherwise, when the test input is from F-MNIST, we observe trends similar to the one reported for the joint training case in  Sec.~\ref{sec:joint}.
Examples for the correctly classified cases can be found in Appendix \ref{app:continual}.

\subsection{Language Modelling}
\label{sec:lm}
The experiments above focus on image data and feedforward NNs.
Here we conduct experiments with texts using a recurrent NN (RNN).
A desirable NLP experiment would be to train a big Transformer language model (LM) on a vast amount of data, and analyse the attention of linear layers over the training datapoints while sampling from the model.
As we can not afford such a gigantic experiment,
we train a one-layer LSTM LM on two small datasets: a tiny public domain book, ``Aesop's Fables'', by J. H. Stickney and the standard WikiText-2 dataset \citep{merity2016pointer}.
The basic approach is similar to the one we introduced above:
we train LMs as usual, while recording inputs to each linear layer.
Here we focus on the linear layer in the LSTM layer (i.e., the single linear transformation which groups all projections including all transformations for the gates).
At test time, we forward the trained LM on a prompt (short text segment) and get the test query from the last token from the prompt.
We compute the attention weights for all (token-level) training datapoints.
Since it is difficult to visualise attention ``distribution'' over different token positions in the entire training tokens in a comprehensive manner, we visualise the text segment around the training datapoints achieving the highest attention weights (we sum all attention weights which go to the same training token position).

Table \ref{tab:wiki2} shows an example from the word-level WikiText-2 experiment.
The test query is from a passage in the Wikipedia page on a warship.
We observe that the training passages achieving the highest attention scores
are also from pages about warships,
and they have a similar sentence structure.
A manual inspection shows that the passage with a similar sentence structure but on an unrelated topic (singer) is not attended, which indicates that the attention here is contextual.
We refer to \textbf{Appendix \ref{app:lm}} for more examples, details, and character-level experiments.
While the scale is limited,
we found all examples interesting and intuitive.
They may be capturing
the essence of how bigger models could be “attending” to different parts of training texts based on some common concept defined by the test query.

\begin{table*}[h!]
\caption{Example test query and top-3 training passages (with their Wikipedia page title) from WikiText-2.
We also show a manually found negative example (which is not among the top scoring passages but has a similar sentence structure).
The test query token and the top scoring training token are highlighted in \bluet{bold}. The query is from the test text.
We refer to \textbf{Appendix \ref{app:lm}} for more examples.
}
% \vspace{-2mm}
\label{tab:wiki2}

\begin{center}
\begin{tabular}{l|l}
\toprule
% https://en.wikipedia.org/wiki/Ironclad_warship
Query (\textit{Ironclad warship}) & ... Her \bluet{principal} role was for combat in the English Channel and other European ...\\
 \midrule
 % https://en.wikipedia.org/wiki/Portuguese_ironclad_Vasco_da_Gama
Top-1 (\textit{Portuguese ironclad}) & Her sailing rig also was removed . Her \bluet{main} battery guns were replaced with new ... \\
% \midrule
% https://en.wikipedia.org/wiki/SMS_Markgraf
Top-2 (\textit{SMS Markgraf}) & ... between the two funnels . Her \bluet{secondary} armament consisted of fourteen 15 cm ...\\
% Top-2 (\textit{SMS Markgraf}) & two superfiring turrets each fore and aft and one turret amidships between the two \\
% &  funnels . Her \bluet{secondary} armament consisted of fourteen 15 cm ...\\
%  \midrule
 % https://en.wikipedia.org/wiki/Italian_cruiser_Aretusa
Top-3 (\textit{Italian cruiser Aretusa})  & single mounts . Her \bluet{primary} offensive weapon was her five 450 mm.\\ 
 \midrule
% https://en.wikipedia.org/wiki/Nina_Simone
Negative (\textit{Nina Simone})  & ... written especially for the singer . Her \textbf{first} hit song in America was her rendition ... \\
\bottomrule
\end{tabular}
\end{center}
\end{table*}

\section{Discussion and Limitations}
\label{sec:limit}
The dual formulation
allows for explicitly visualising attention weights over all training patterns, given a test input.
While we argue that this view provides a new perspective on analysing neural networks,
it also has several limitations.
First, the memory storage requirement forces us to conduct experiments with small datasets (see more discussion in Appendix \ref{app:scalability}).
On the other hand, storage requirements grow linearly with training set size, 
while computing hardware is still getting exponentially cheaper with time.
That is, soon we may be able to analyse  much larger models trained on much larger datasets.
Second, our analysis is not applicable to models which are already trained.
Furthermore, it is limited to a study of attention weights, in line with traditional visualisations of attention-based systems,  and can only show which training datapoints are combined.
It does not tell how the combined representations can be converted to a meaningful output, e.g., in the case of large generative NNs  mentioned in the introduction. 

\section{Conclusion}
We revisit the dual form of the perceptron for linear layers in deep NNs.
The dual form is expressed in terms of the now popular key/value/attention concepts,
which offers novel insights and interpretability. 
We visualise and study the corresponding attention weights.
This allows for connecting training datapoints to test time predictions,
and observing many interesting patterns in various scenarios, on both image and language modalities.
While our analysis is still limited to relatively small datasets,
it opens up new avenues for analysing and interpreting behaviours of deep NNs. 

\section*{Acknowledgements}
This research was partially funded by ERC Advanced grant 742870, project AlgoRNN,
and by Swiss National Science Foundation grant 200021\_192356, project NEUSYM.
We are thankful for hardware donations from NVIDIA \& IBM. The resources used here were partially provided by Swiss National Supercomputing Centre (CSCS) project s1023.

\bibliography{paper}
\bibliographystyle{icml2022}

%%%%%%%%%%%%%%%%%%%%%%%%%%%%%%%%%%%%%%%%%%%%%%%%%%%%%%%%%%%%%%%%%%%%%%%%%%%%%%%
%%%%%%%%%%%%%%%%%%%%%%%%%%%%%%%%%%%%%%%%%%%%%%%%%%%%%%%%%%%%%%%%%%%%%%%%%%%%%%%
% APPENDIX
%%%%%%%%%%%%%%%%%%%%%%%%%%%%%%%%%%%%%%%%%%%%%%%%%%%%%%%%%%%%%%%%%%%%%%%%%%%%%%%
%%%%%%%%%%%%%%%%%%%%%%%%%%%%%%%%%%%%%%%%%%%%%%%%%%%%%%%%%%%%%%%%%%%%%%%%%%%%%%%
\newpage
\appendix
\onecolumn

% You can have as much text here as you want. The main body must be at most $8$ pages long.
% For the final version, one more page can be added.
% If you want, you can use an appendix like this one, even using the one-column format.
%%%%%%%%%%%%%%%%%%%%%%%%%%%%%%%%%%%%%%%%%%%%%%%%%%%%%%%%%%%%%%%%%%%%%%%%%%%%%%%
%%%%%%%%%%%%%%%%%%%%%%%%%%%%%%%%%%%%%%%%%%%%%%%%%%%%%%%%%%%%%%%%%%%%%%%%%%%%%%%

\section{Top Matching Examples}
\label{app:top}
This section contains figures of
images in the training set achieving the highest attention scores (top-3) in each layer in the single task and multi-task joint training cases.
We refer to these examples in the main text in Sec.~\ref{sec:single} and \ref{sec:joint}.
We note that one example from class 0 (Figure \ref{sfig:mnist_layer0_top1}) is often found in the top-3 matches in different settings we study, presumably as the corresponding image of digit ``0'' highly overlaps with other images in the original image space.

%%%%%%%%%% Single task case
\begin{figure*}[ht]
    \centering
	\subfloat[layer-0, top-1]{
		\centering
		\includegraphics[width=.15\linewidth]{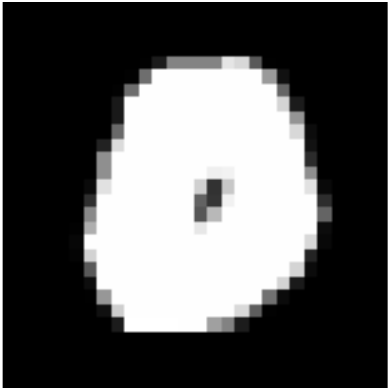}
		\label{sfig:mnist_layer0_top1}
	}
	\subfloat[layer-1, top-1]{
		\centering
		\includegraphics[width=.15\linewidth]{figures/single6_zero.pdf}
		\label{sfig:mnist_layer1_top1}
	}
	\subfloat[layer-1, top-2]{
		\centering
		\includegraphics[width=.15\linewidth]{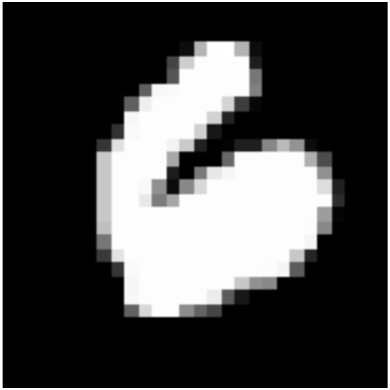}
		\label{sfig:mnist_layer1_top2}
	}
	\subfloat[layer-2, top-1]{
		\centering
		\includegraphics[width=.15\linewidth]{figures/single6_zero.pdf}
		\label{sfig:mnist_layer2_top1}
	}
	\subfloat[layer-2, top-2]{
		\centering
		\includegraphics[width=.15\linewidth]{figures/single6_layer1_top2.pdf}
		\label{sfig:mnist_layer2_top2}
	}
	\subfloat[layer-2, top-3]{
		\centering
		\includegraphics[width=.15\linewidth]{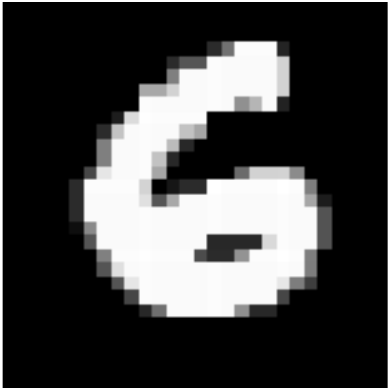}
		\label{sfig:mnist_layer2_top2}
	}
	%\vspace{-3mm}
	\caption{\it Top matching training examples 	for the input test example from class 6 (Figure \ref{sfig:input_mnist})
	in the single task case on MNIST.}
	\label{fig:mnist_top_matching}
		\vspace{-3mm}
\end{figure*}

%%%%%%%%%% Joint training case
\begin{figure*}[ht]
\centering
	\subfloat[layer-0, top-1]{
		\centering
		\includegraphics[width=.15\linewidth]{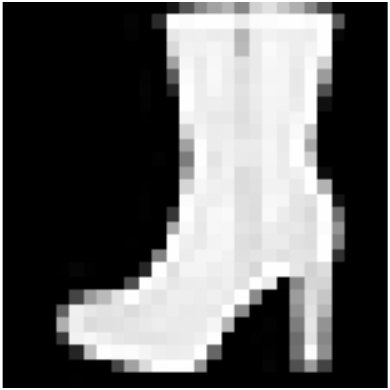}
		\label{sfig:f_mnist_layer0_top1}
	}
	\subfloat[layer-1, top-1]{
		\centering
		\includegraphics[width=.15\linewidth]{figures/single6_zero.pdf}
		\label{sfig:f_mnist_layer1_top1}
	}
	\subfloat[layer-1, top-2]{
		\centering
		\includegraphics[width=.15\linewidth]{figures/booth_match.pdf}
		\label{sfig:f_mnist_layer1_top2}
	}
	\subfloat[layer-2, top-2]{
		\centering
		\includegraphics[width=.15\linewidth]{figures/single6_zero.pdf}
		\label{sfig:f_mnist_layer2_top2}
	}
	\subfloat[layer-2, top-3]{
		\centering
		\includegraphics[width=.15\linewidth]{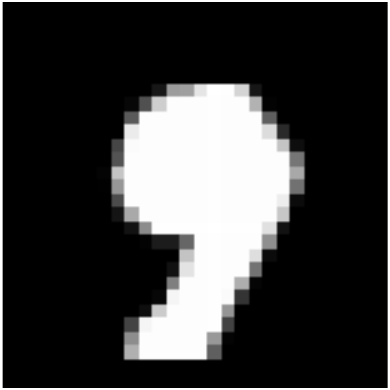}
		\label{sfig:f_mnist_layer2_top2}
	}
	%\vspace{-3mm}
	\caption{\it Top matching training examples.
		for the input test example from class 9 of F-MNIST (Figure \ref{sfig:input_f_mnist})
	in the joint training case.}
	\label{fig:f_mnist_top_matching}
		\vspace{-3mm}
\end{figure*}

%%%%%%%%%%%%%%%%%%%%%%%%%%%%%%%%%%%%%%%%%%%%
%%%%%%%%%%%%%%%%%%%%%%%%%%%%%%%%%%%%%%%%%%%%

\section{More Examples/Visualisation}
\label{app:examples}

Here we provide more examples that do not fit the space limitations  of the main text.
All input test examples used to generate figures in the appendix are summarised in Figure \ref{fig:app_all_input}.
We refer to each of them in the subsections below.

%%%%%%%%%% All examples
\begin{figure*}[ht]
	\subfloat[MNIST, 3]{
		\centering
		\includegraphics[width=.15\linewidth]{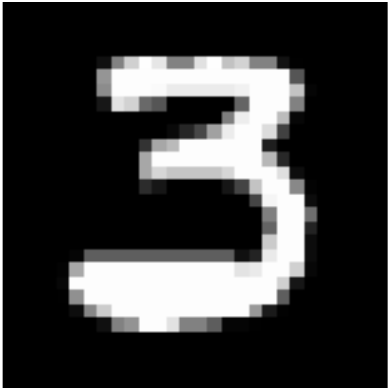}
		\label{sfig:app_input_single}
	}
	\subfloat[F-MNIST, 3 (dress)]{
		\centering
		\includegraphics[width=.15\linewidth]{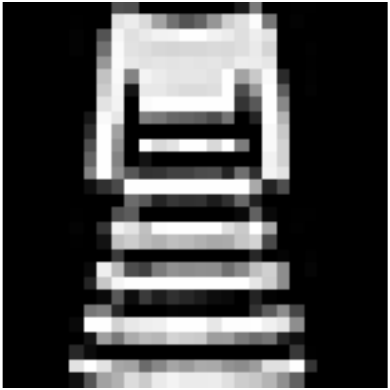}
		\label{sfig:app_input_joint_dress}
	}
	\hspace{1mm}
	\subfloat[F-MNIST, 5 (sandal)]{
		\centering
		\hspace{1mm}\includegraphics[width=.15\linewidth]{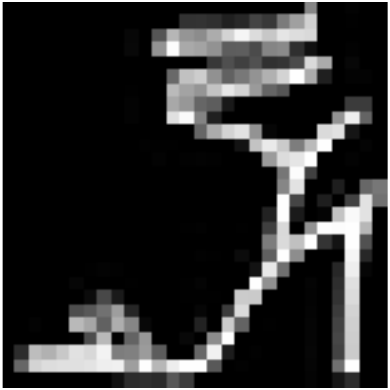}\hspace{1mm}
		\label{sfig:app_input_joint_sandal}
	}
	\subfloat[MNIST, 1]{
		\centering
		\includegraphics[width=.15\linewidth]{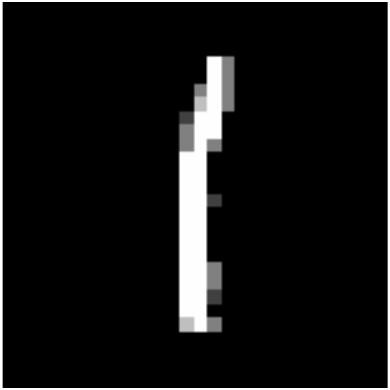}
		\label{sfig:app_input_seq_misclassified}
	}
	\subfloat[F-MNIST, 3 (dress)]{
		\centering
		\includegraphics[width=.15\linewidth]{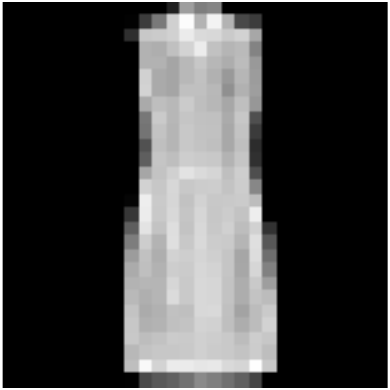}
		\label{sfig:app_input_seq_fmnist}
	}
	\subfloat[MNIST, 7]{
		\centering
		\includegraphics[width=.15\linewidth]{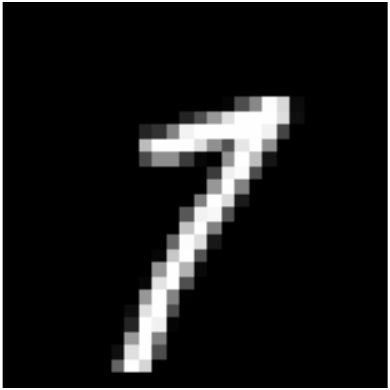}
		\label{sfig:app_input_seq_mnist}
	}
	%\vspace{-3mm}
	\caption{\it Input test examples used for various scenarios studied in this appendix.}
	\label{fig:app_all_input}
		\vspace{-3mm}
\end{figure*}

\subsection{Single Task Case}
\label{app:single}

Here we simply show one more example in the single task case discussed in Sec.~\ref{sec:single}.
The input test example used here is shown in Figure \ref{sfig:app_input_single} (digit ``3'').
The plots of attention weights and the corresponding sums per class are show in Figures \ref{fig:app_mnist_heatmaps} and \ref{fig:app_mnist_per_class_sum}.
The observations are similar to those discussed in Sec.~\ref{sec:single}.
The training examples achieving the highest attention scores
in each layer are shown in Figure \ref{app:fig:single}.

\begin{figure*}[ht]
\centering
    \subfloat[layer-0 top-1]{
        \centering
        \includegraphics[width=.2\linewidth]{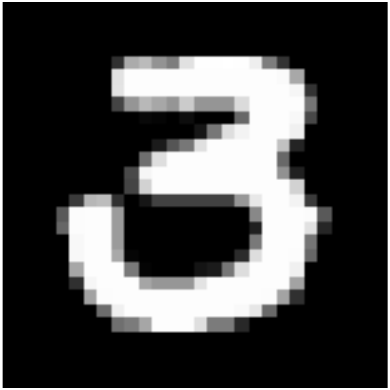}
    }
    \subfloat[layer-0 top-2]{
        \centering
        \includegraphics[width=.2\linewidth]{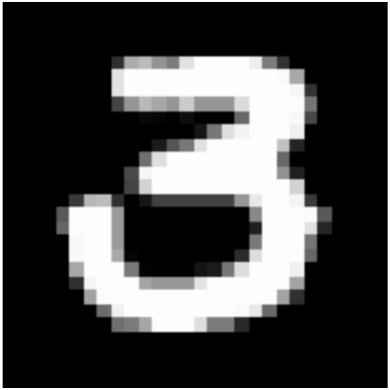}
    }
    \subfloat[layer-0 top-3]{
        \centering
        \includegraphics[width=.2\linewidth]{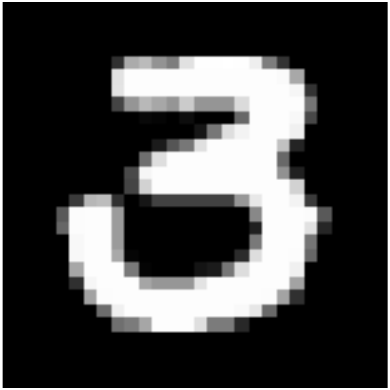}
    }
\\
    \subfloat[layer-1 top-1]{
        \centering
        \includegraphics[width=.2\linewidth]{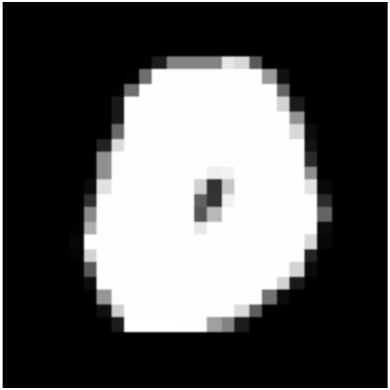}
    }
    \subfloat[layer-1 top-2]{
        \centering
        \includegraphics[width=.2\linewidth]{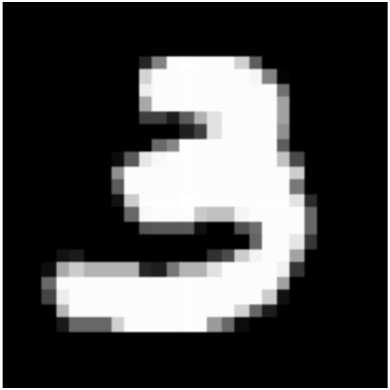}
    }
    \subfloat[layer-1 top-3]{
        \centering
        \includegraphics[width=.2\linewidth]{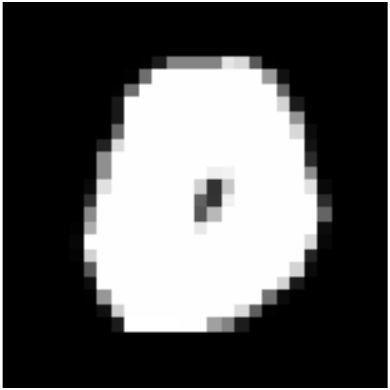}
    }
\\
    \subfloat[layer-2 top-1]{
        \centering
        \includegraphics[width=.2\linewidth]{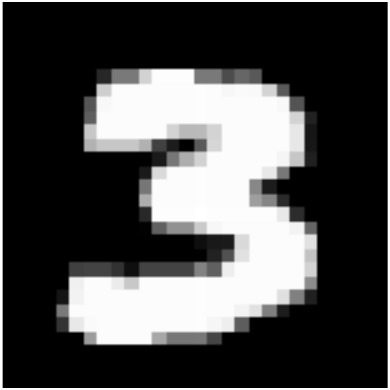}
    }
    \subfloat[layer-2 top-2]{
        \centering
        \includegraphics[width=.2\linewidth]{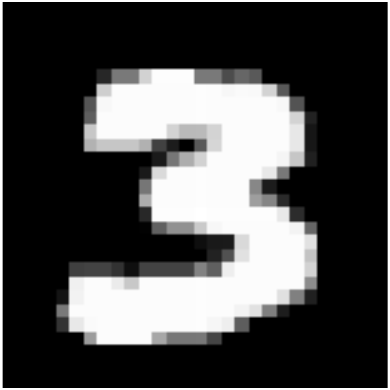}
    }
    \subfloat[layer-2 top-3]{
        \centering
        \includegraphics[width=.2\linewidth]{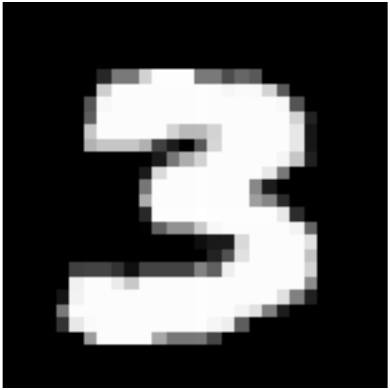}
    }
    \caption{Top scoring examples for the MNIST input of Figure \ref{sfig:app_input_single} in the single task case.}
    \label{app:fig:single}
\end{figure*}

\begin{figure*}[ht]
	\subfloat[layer-0]{
		\centering
		\includegraphics[width=.32\linewidth]{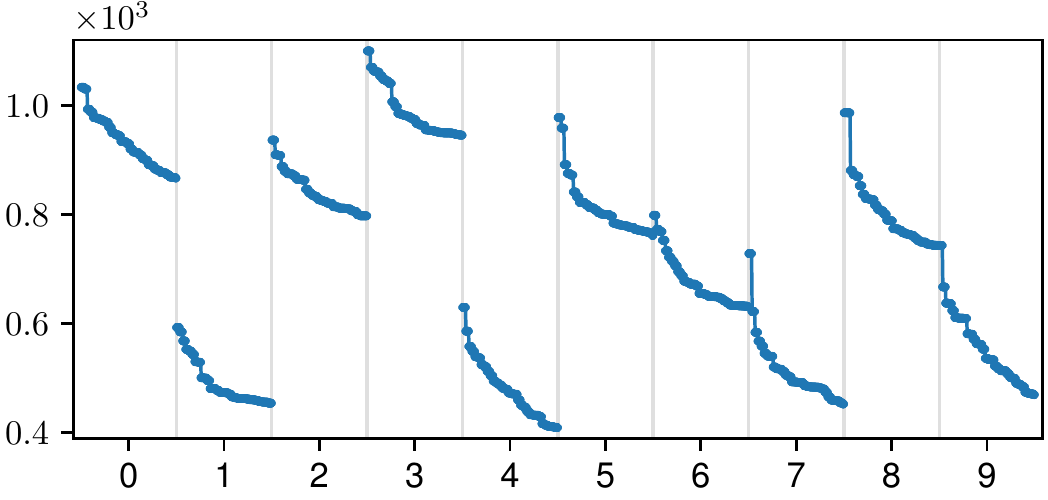}
% 		\label{sfig:f_mnist_layer0}
	}
	\subfloat[layer-1] {
		\centering
		\includegraphics[width=.32\linewidth]{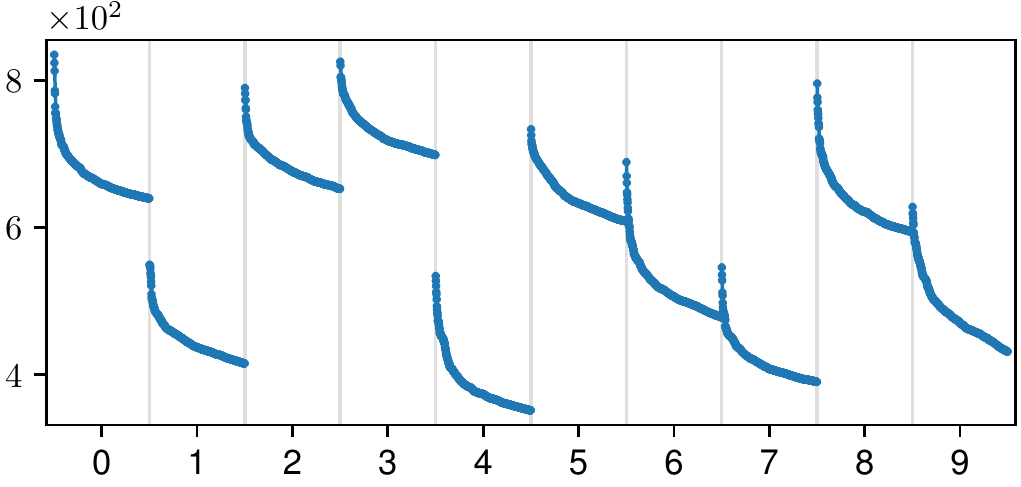}
% 		\label{sfig:f_mnist_layer1}
	}
	\subfloat[layer-2]{
		\centering
		\includegraphics[width=.32\linewidth]{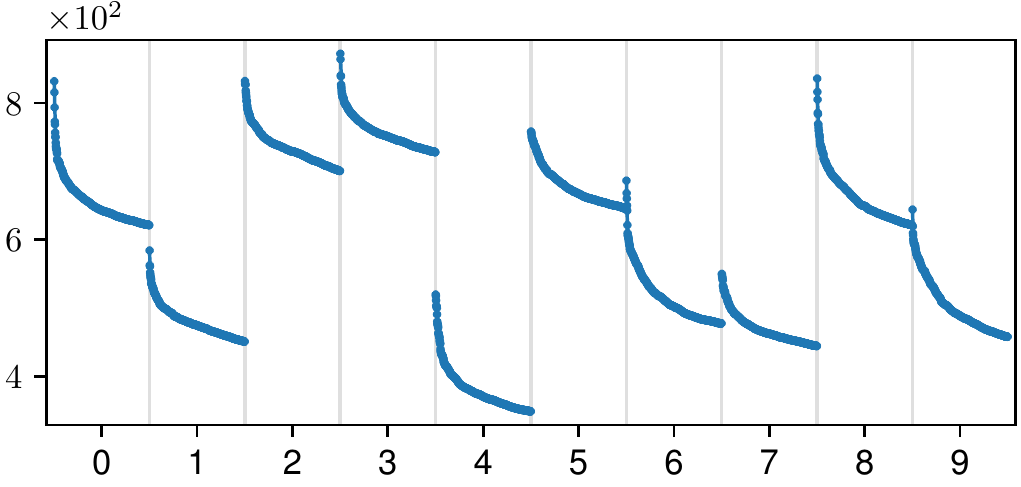}
% 		\label{sfig:f_mnist_layer2}
	}
	%\vspace{-3mm}
	\caption{\it Attention weights over training examples for the input test example from class 3 of MNIST (Figure \ref{sfig:app_input_single}) in the \textbf{single task case}. x-axis is partitioned by class, and for each class, top-500 datapoints sorted in descending order are shown.}
 	\label{fig:app_mnist_heatmaps}
		\vspace{-3mm}
\end{figure*}

\begin{figure*}[ht]
	\subfloat[layer-0]{
		\centering
		\includegraphics[width=.32\linewidth]{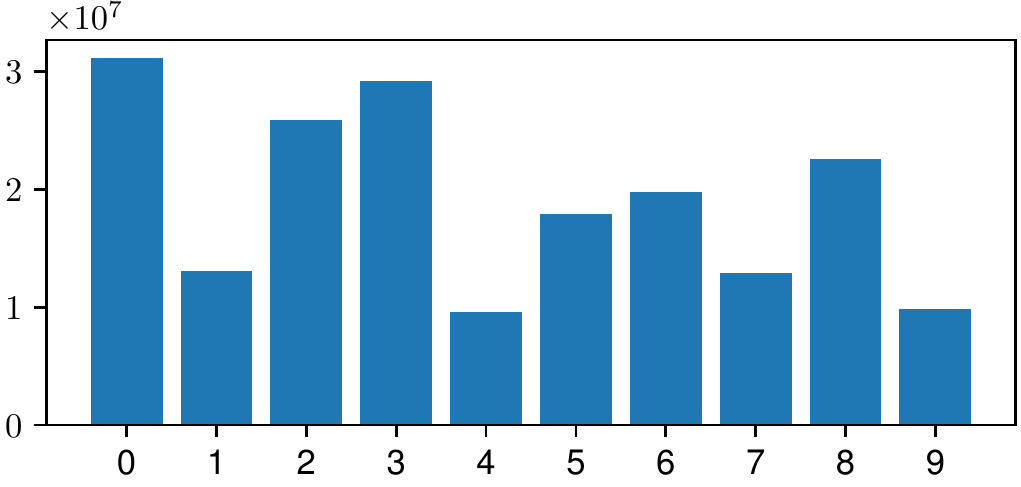}
% 		\label{sfig:f_mnist_layer0_sum}
	}
	\subfloat[layer-1] {
		\centering
		\includegraphics[width=.32\linewidth]{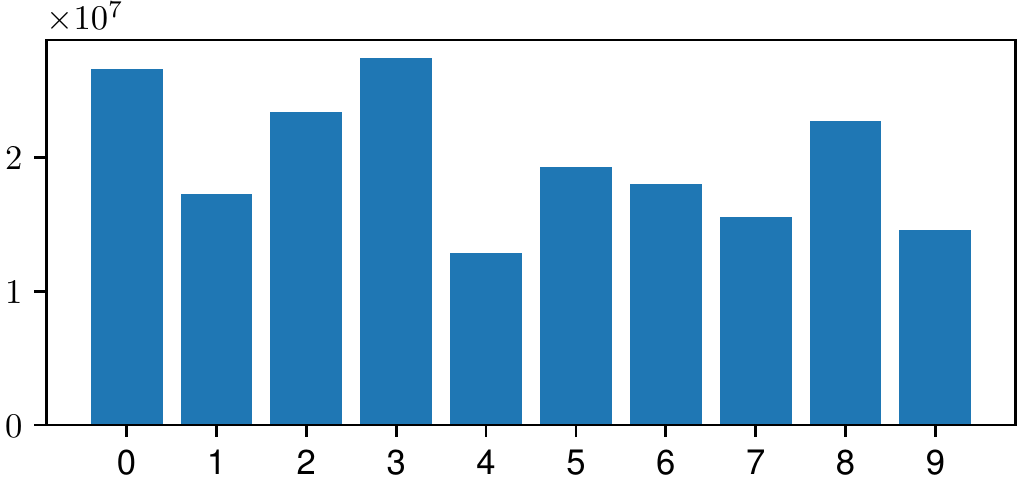}
% 		\label{sfig:f_mnist_layer1_sum}
	}
	\subfloat[layer-2]{
		\centering
		\includegraphics[width=.32\linewidth]{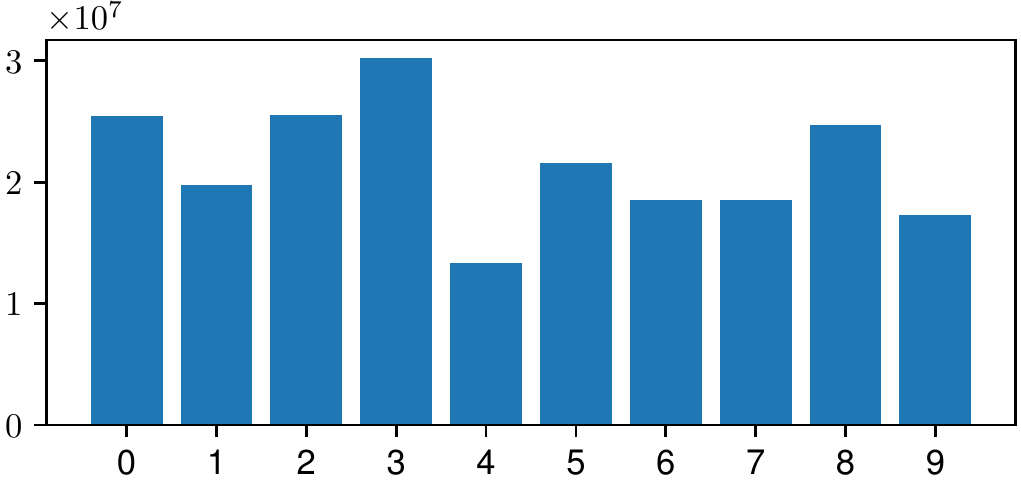}
% 		\label{sfig:f_mnist_layer2_sum}
	}
	%\vspace{-3mm}
	\caption{\it Total scores per class,
	for the input test example from class 3 of MNIST (Figure \ref{sfig:app_input_single})
	in the \textbf{single task case}.}
 	\label{fig:app_mnist_per_class_sum}
		\vspace{-3mm}
\end{figure*}

\subsection{Joint Training Case}
\label{app:joint}

Here we show two more examples for the multi-task joint training case discussed in Sec.~\ref{sec:joint}.
The input test examples used here are shown in Figure \ref{sfig:app_input_joint_dress} (class 3 ``dress'') and \ref{sfig:app_input_joint_sandal} (class 5 ``sandal'').
Both of them are correctly classified by the model.
The training examples achieving the highest attention scores
in each layer are shown in Figures \ref{app:fig:joint_dress}
and \ref{app:fig:joint_sandal}, respectively.
The plots of the attention weights and the corresponding sums per class are shown in Figures \ref{fig:app_f_mnist_joint_0} and \ref{fig:app_sum_f_mnist_joint_0} for the class 3 input,
and in Figures \ref{fig:app_f_mnist_joint_1} and \ref{fig:app_sum_f_mnist_joint_1} for the class 5 input.
In the first example (F-MNIST, class 3), we do not see MNIST examples among the top scoring training examples (see Figure \ref{app:fig:joint_dress}), unlike in the second case (F-MNIST, class 5) where we do see the digit ``5'' of MNIST (see Figure \ref{app:fig:joint_sandal}).
The second example also illustrates the case where the argmax of the sum disagrees with the model decision
(illustrating the limitations of analysis based on attention weights only, see Sec.~\ref{sec:limit}).

\begin{figure*}[ht]
\centering
    \subfloat[layer-0 top-1]{
        \centering
        \includegraphics[width=.2\linewidth]{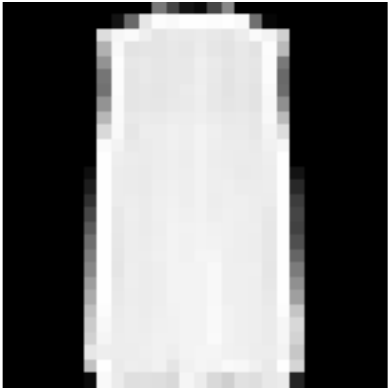}
    }
    \subfloat[layer-0 top-2]{
        \centering
        \includegraphics[width=.2\linewidth]{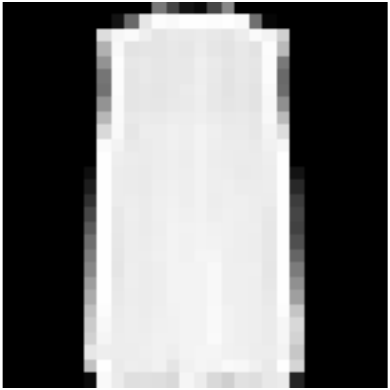}
    }
    \subfloat[layer-0 top-3]{
        \centering
        \includegraphics[width=.2\linewidth]{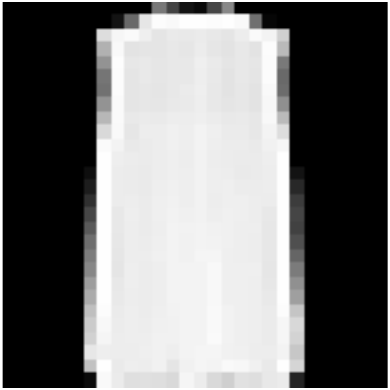}
    }
\\
    \subfloat[layer-1 top-1]{
        \centering
        \includegraphics[width=.2\linewidth]{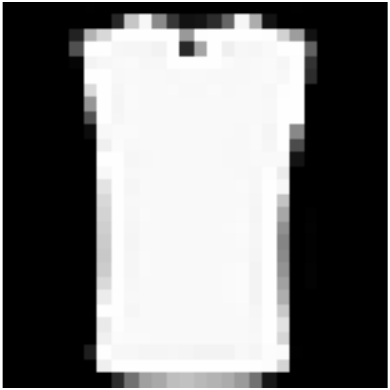}
    }
    \subfloat[layer-1 top-2]{
        \centering
        \includegraphics[width=.2\linewidth]{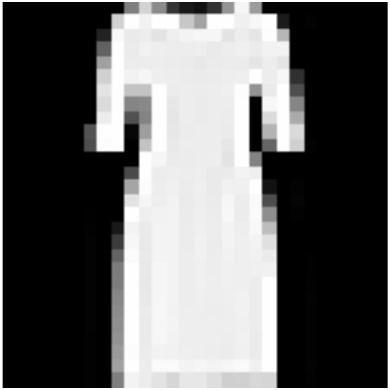}
    }
    \subfloat[layer-1 top-3]{
        \centering
        \includegraphics[width=.2\linewidth]{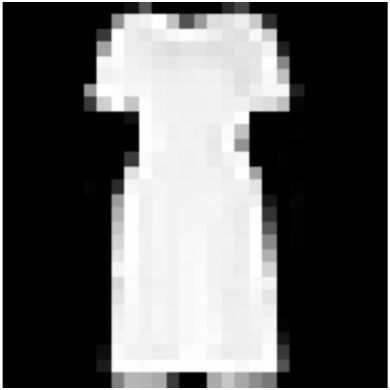}
    }
\\
    \subfloat[layer-2 top-1]{
        \centering
        \includegraphics[width=.2\linewidth]{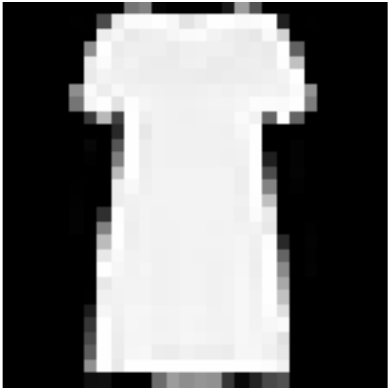}
    }
    \subfloat[layer-2 top-2]{
        \centering
        \includegraphics[width=.2\linewidth]{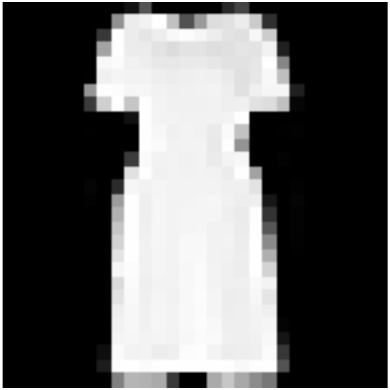}
    }
    \subfloat[layer-2 top-3]{
        \centering
        \includegraphics[width=.2\linewidth]{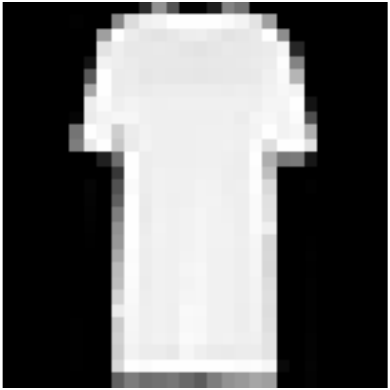}
    }
    \caption{Top scoring examples for the F-MNIST input of Figure \ref{sfig:app_input_joint_dress} in the joint training case.}
    \label{app:fig:joint_dress}
    % \caption{Index 3 Top examples for mixed\_correct\_1}
\end{figure*}

\begin{figure*}[ht]
	\subfloat[layer-0]{
		\centering
		\includegraphics[width=.32\linewidth]{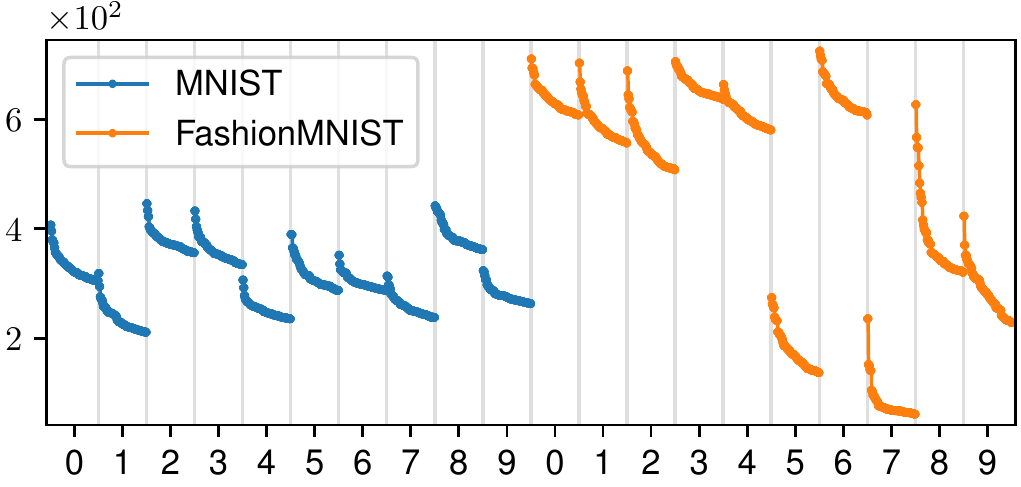}
%		\label{sfig:f_mnist_layer0}
	}
	\subfloat[layer-1] {
		\centering
		\includegraphics[width=.32\linewidth]{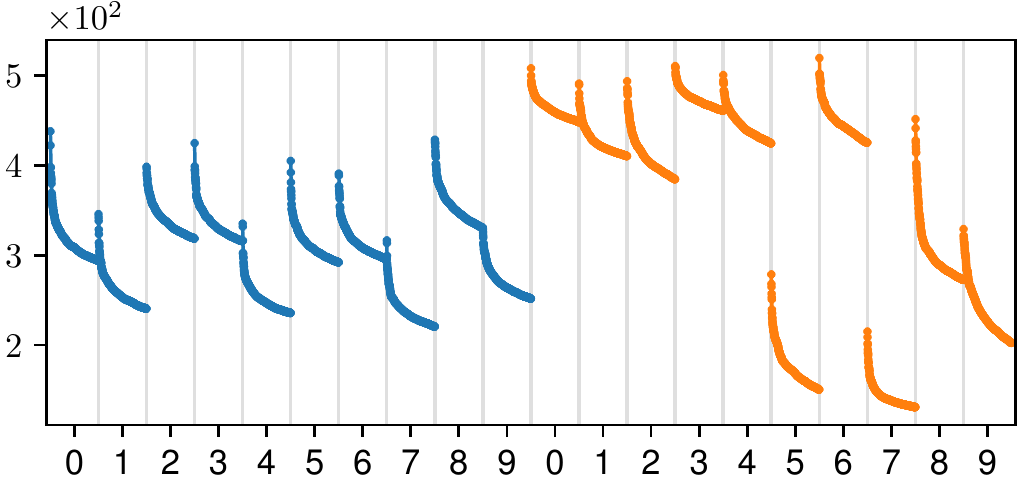}
%		\label{sfig:f_mnist_layer1}
	}
	\subfloat[layer-2]{
		\centering
		\includegraphics[width=.32\linewidth]{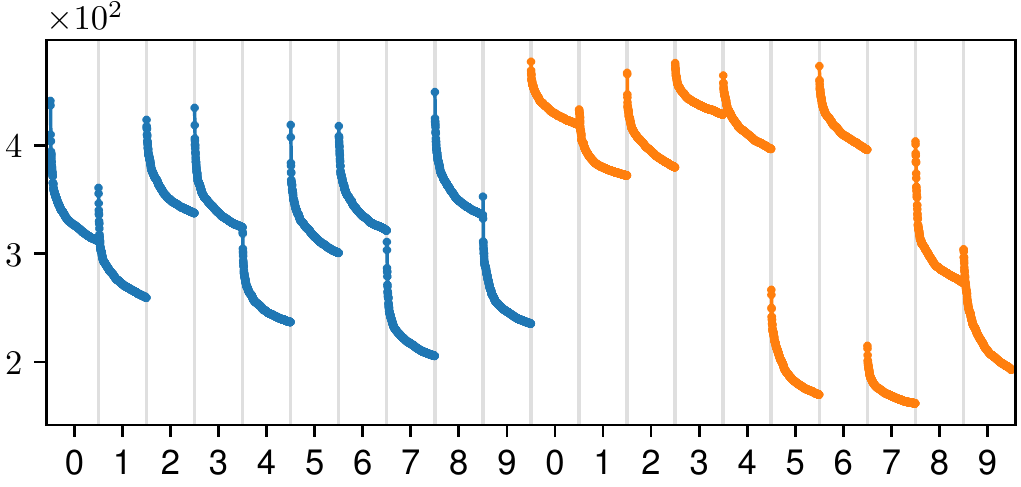}
%		\label{sfig:f_mnist_layer2}
	}
	%\vspace{-3mm}
	% Index3 
	\caption{\it Attention weights over training examples for the input test example from class 3 of F-MNIST (Figure \ref{sfig:app_input_joint_dress}) in the \textbf{joint training case}. The x-axis is partitioned by class, and for each class, top-500 datapoints sorted in descending order are shown.}
	\label{fig:app_f_mnist_joint_0}
		\vspace{-3mm}
\end{figure*}

\begin{figure*}[ht]
	\subfloat[layer-0]{
		\centering
		\includegraphics[width=.32\linewidth]{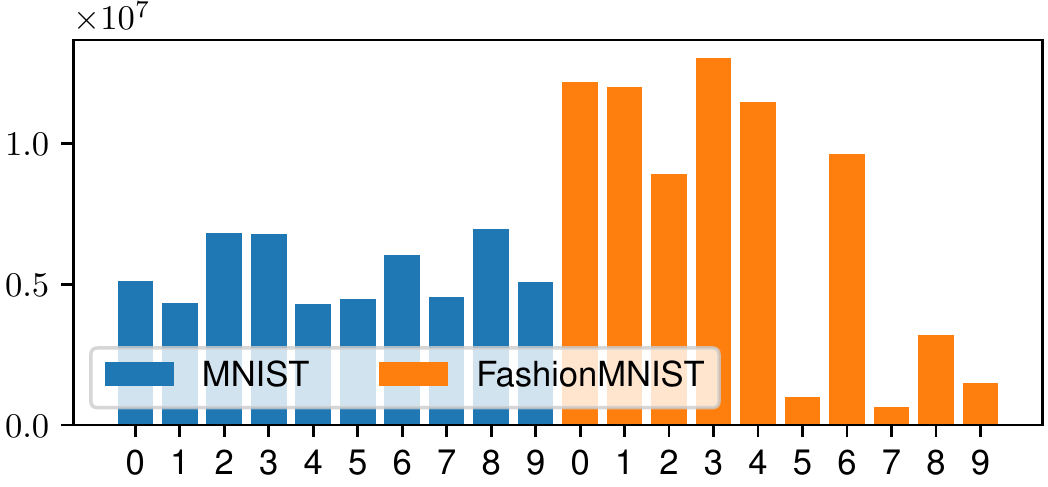}
%		\label{sfig:f_mnist_layer0_sum}
	}
	\subfloat[layer-1] {
		\centering
		\includegraphics[width=.32\linewidth]{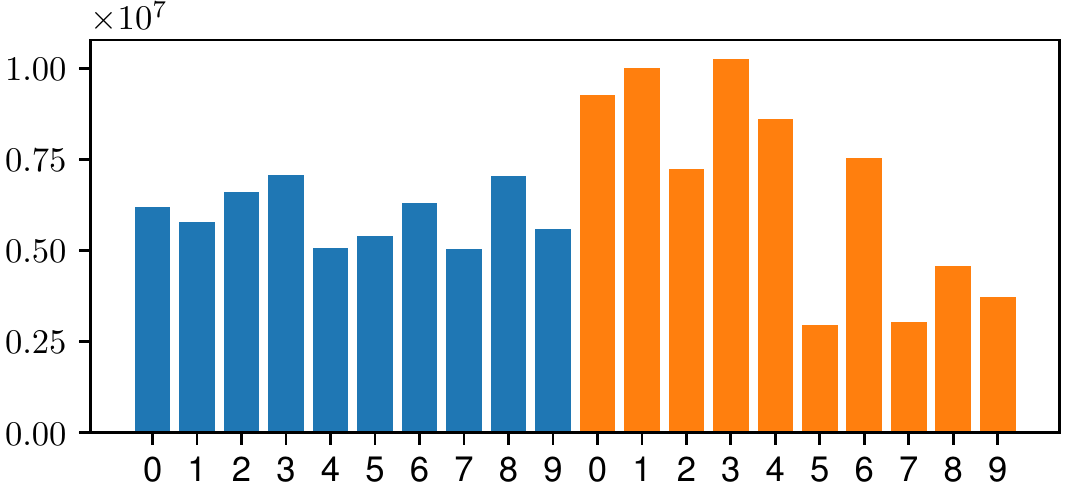}
%		\label{sfig:f_mnist_layer1_sum}
	}
	\subfloat[layer-2]{
		\centering
		\includegraphics[width=.32\linewidth]{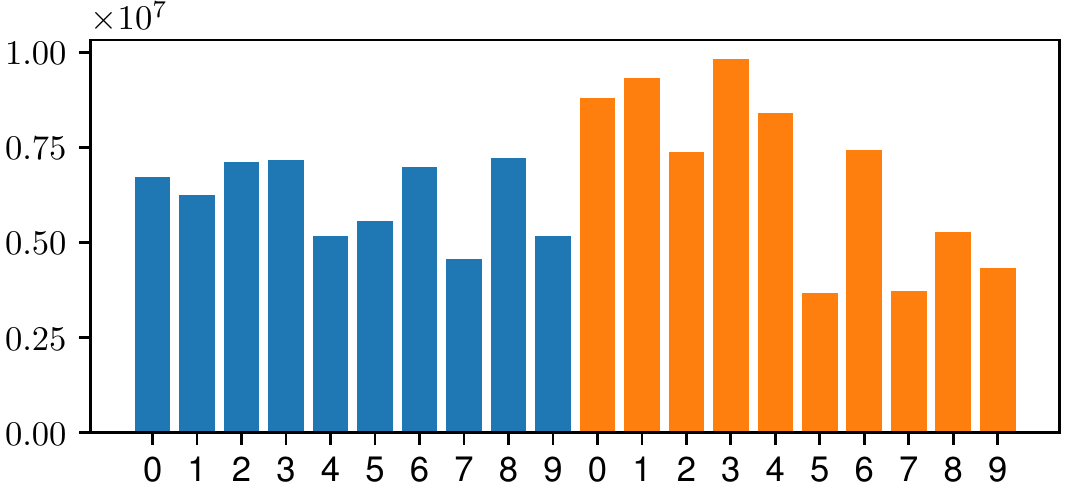}
%		\label{sfig:f_mnist_layer2_sum}
	}
	%\vspace{-3mm}
	% Index 3
	\caption{\it Total scores per class
	for the input test example from class 3 of F-MNIST (Figure \ref{sfig:app_input_joint_dress})
	in the \textbf{joint training case}.}
 	\label{fig:app_sum_f_mnist_joint_0}
		\vspace{-3mm}
\end{figure*}

\begin{figure*}[ht]
\centering
    \subfloat[layer-0 top-1]{
        \centering
        \includegraphics[width=.2\linewidth]{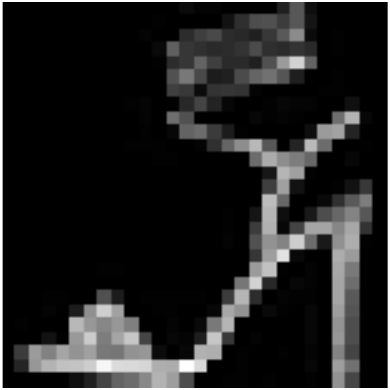}
    }
    \subfloat[layer-0 top-2]{
        \centering
        \includegraphics[width=.2\linewidth]{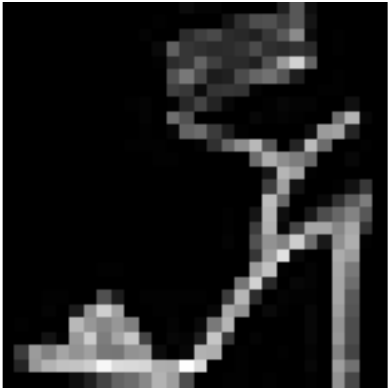}
    }
    \subfloat[layer-0 top-3]{
        \centering
        \includegraphics[width=.2\linewidth]{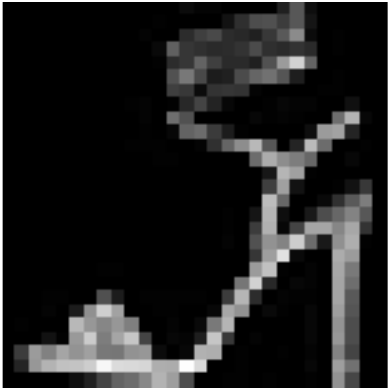}
    }
\\
    \subfloat[layer-1 top-1]{
        \centering
        \includegraphics[width=.2\linewidth]{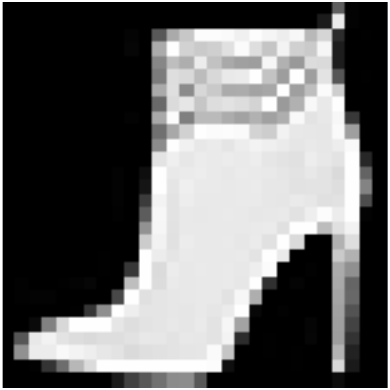}
    }
    \subfloat[layer-1 top-2]{
        \centering
        \includegraphics[width=.2\linewidth]{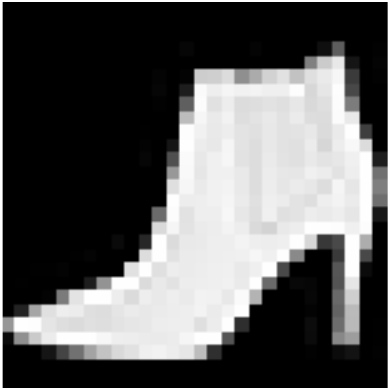}
    }
    \subfloat[layer-1 top-3]{
        \centering
        \includegraphics[width=.2\linewidth]{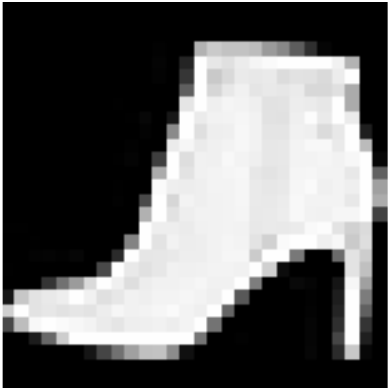}
    }
\\
    \subfloat[layer-2 top-1]{
        \centering
        \includegraphics[width=.2\linewidth]{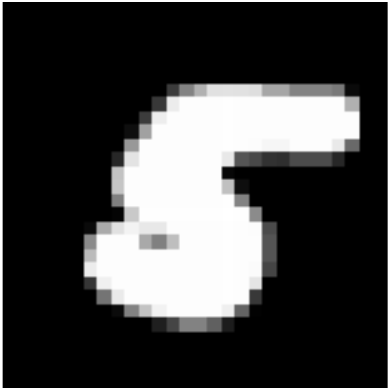}
    }
    \subfloat[layer-2 top-2]{
        \centering
        \includegraphics[width=.2\linewidth]{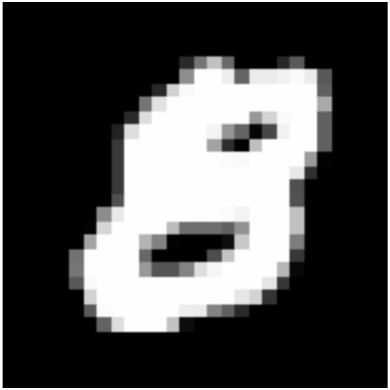}
    }
    \subfloat[layer-2 top-3]{
        \centering
        \includegraphics[width=.2\linewidth]{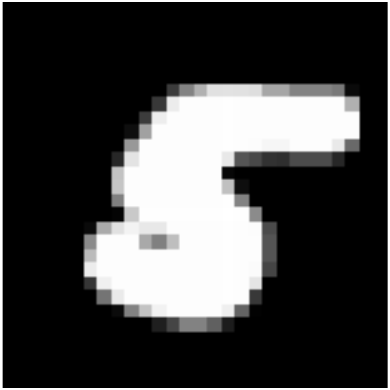}
    }
    \caption{Top scoring examples for the F-MNIST input of Figure \ref{sfig:app_input_joint_sandal} in the joint training case.}
    \label{app:fig:joint_sandal}
%     \caption{Index 4 Top examples for mixed\_correct\_2}
\end{figure*}

\begin{figure*}[ht]
	\subfloat[layer-0]{
		\centering
		\includegraphics[width=.32\linewidth]{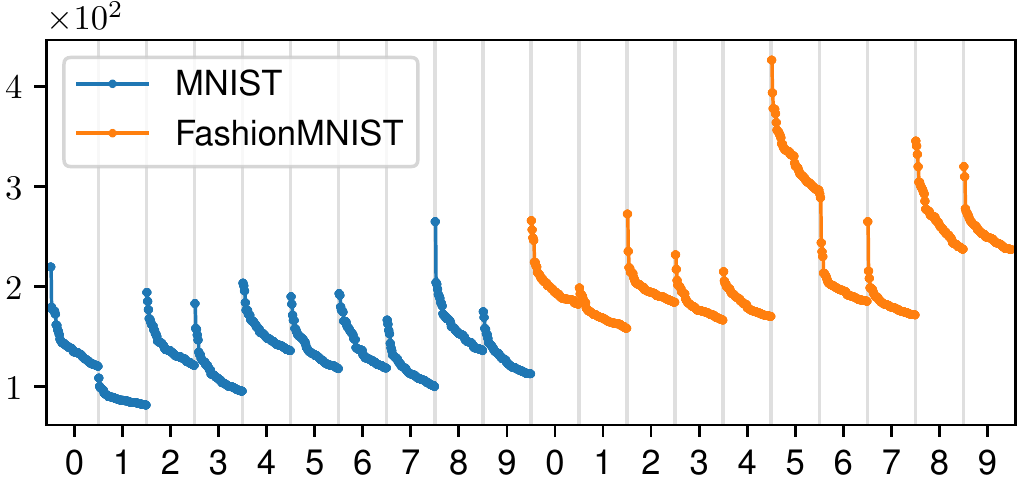}
%		\label{sfig:f_mnist_layer0}
	}
	\subfloat[layer-1] {
		\centering
		\includegraphics[width=.32\linewidth]{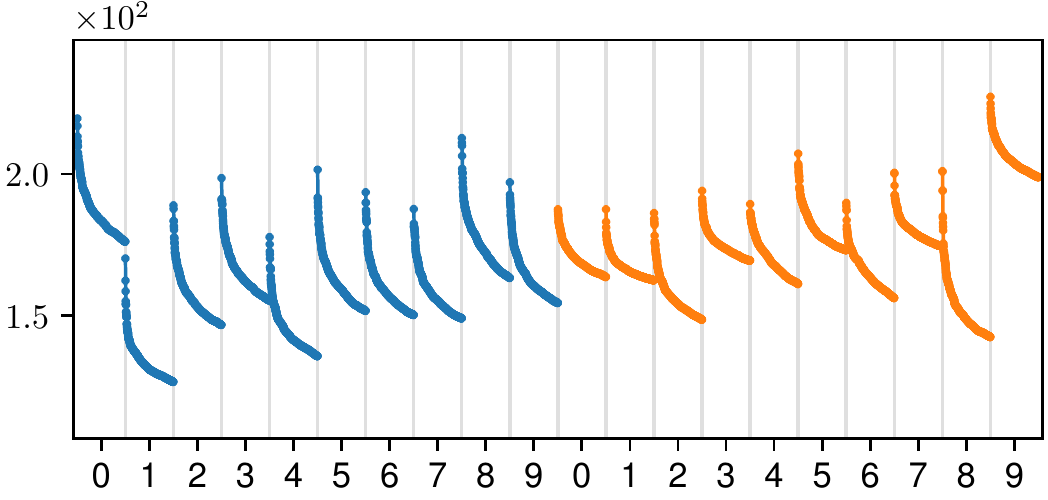}
%		\label{sfig:f_mnist_layer1}
	}
	\subfloat[layer-2]{
		\centering
		\includegraphics[width=.32\linewidth]{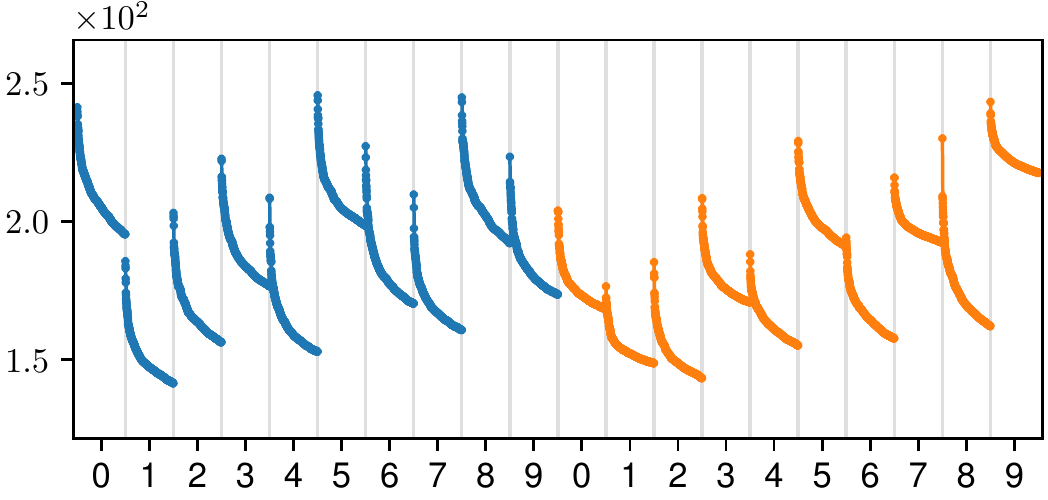}
%		\label{sfig:f_mnist_layer2}
	}
	%\vspace{-3mm}
	% Index4 
	\caption{\it Attention weights over training examples for the input test example from class 5 of F-MNIST (Figure \ref{sfig:app_input_joint_sandal}) in the joint training case. The x-axis is partitioned by class, and for each class, top-500 datapoints sorted in descending order are shown.}
 	\label{fig:app_f_mnist_joint_1}
		\vspace{-3mm}
\end{figure*}

\begin{figure*}[ht]
	\subfloat[layer-0]{
		\centering
		\includegraphics[width=.32\linewidth]{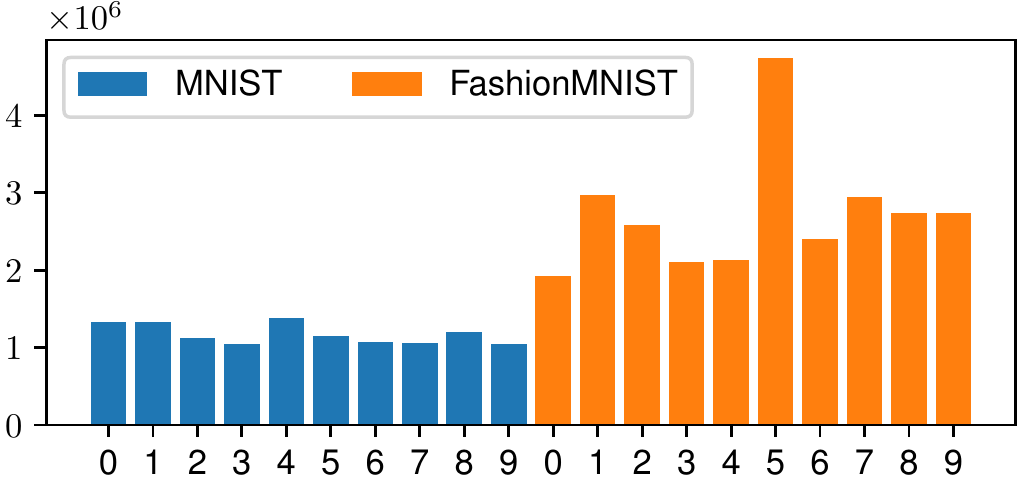}
%		\label{sfig:f_mnist_layer0_sum}
	}
	\subfloat[layer-1] {
		\centering
		\includegraphics[width=.32\linewidth]{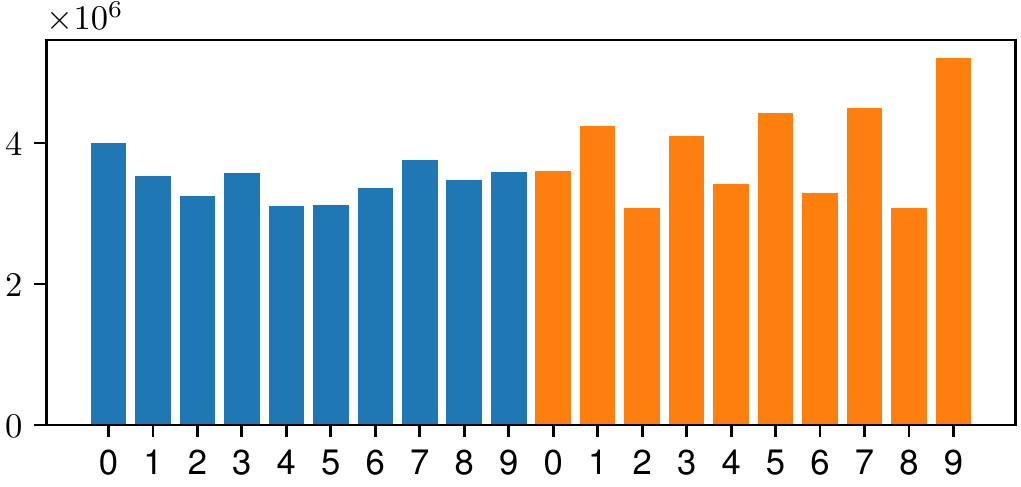}
%		\label{sfig:f_mnist_layer1_sum}
	}
	\subfloat[layer-2]{
		\centering
		\includegraphics[width=.32\linewidth]{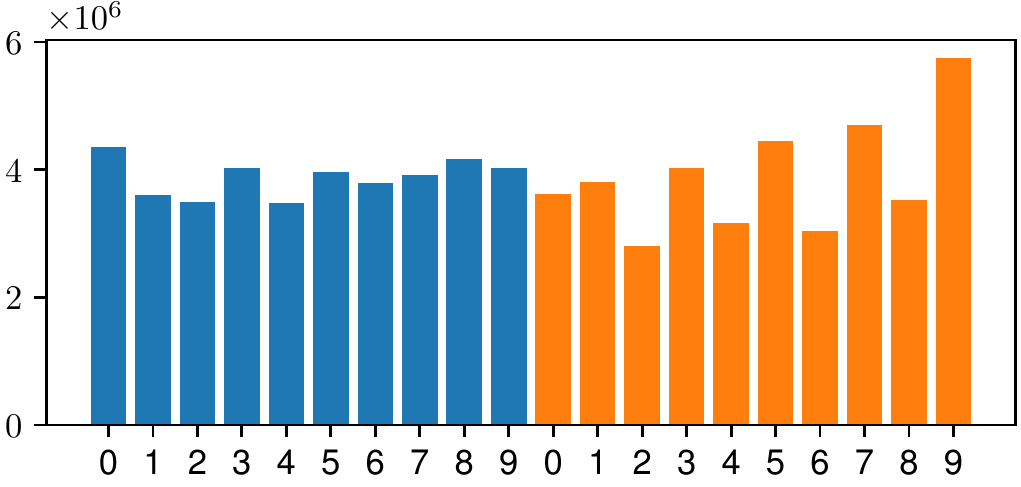}
%		\label{sfig:f_mnist_layer2_sum}
	}
	%\vspace{-3mm}
	% Index 4
	\caption{\it Total scores per class
	for the input test example from class 5 of F-MNIST (Figure \ref{sfig:app_input_joint_sandal})
	in the joint training case.}
 	\label{fig:app_sum_f_mnist_joint_1}
		\vspace{-3mm}
\end{figure*}

\subsection{Continual Learning Case}
\label{app:continual}

In the continual learning case, we show three examples:
one MNIST example (Figure \ref{sfig:app_input_seq_misclassified}) which is misclassified,
and two others which are correctly classified (one from MNIST is shown in Figure \ref{sfig:app_input_seq_mnist}, one from F-MNIST  in Figure \ref{sfig:app_input_seq_fmnist}).

The misclassified example (Figures \ref{fig:mis_mnist_heatmaps} and \ref{fig:mis_mnist_sum}) has been already discussed in Sec.~\ref{sec:continual}.

For the correctly classified inputs,
Figures \ref{fig:app_mnist_seq} and \ref{fig:app_mnist_seq_sum} respectively show the attention weights and the sums for the MNIST input (Figure \ref{sfig:app_input_seq_mnist}) of class 7,
and
Figures \ref{fig:app_fmnist_seq} and \ref{fig:app_fmnist_seq_sum} show them for the F-MNIST input (Figure \ref{sfig:app_input_seq_fmnist}) of class 3 (``dress'').

The training examples achieving the highest attention scores
in each case are shown in Figures \ref{app:fig:misclassified_mnist},
\ref{app:fig:seq_f_mnist}, and \ref{app:fig:seq_mnist}.

\begin{figure*}[ht]
\centering
    \subfloat[layer-0 top-1]{
        \centering
        \includegraphics[width=.2\linewidth]{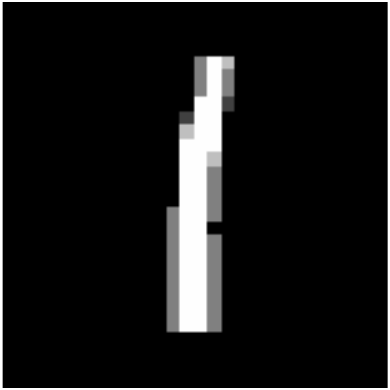}
    }
    \subfloat[layer-0 top-2]{
        \centering
        \includegraphics[width=.2\linewidth]{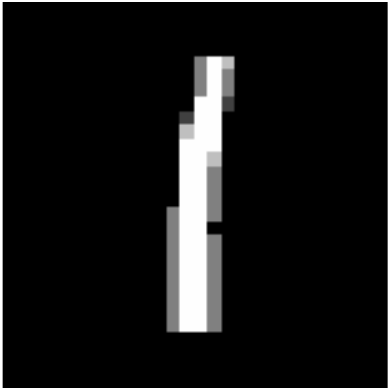}
    }
    \subfloat[layer-0 top-3]{
        \centering
        \includegraphics[width=.2\linewidth]{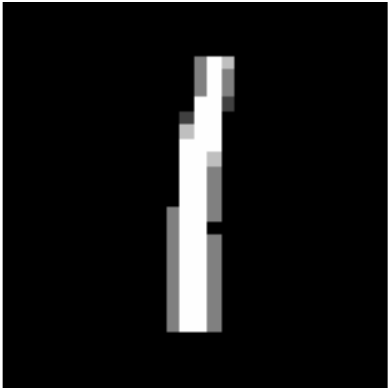}
    }
\\
    \subfloat[layer-1 top-1]{
        \centering
        \includegraphics[width=.2\linewidth]{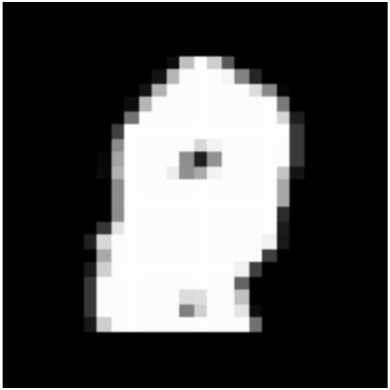}
    }
    \subfloat[layer-1 top-2]{
        \centering
        \includegraphics[width=.2\linewidth]{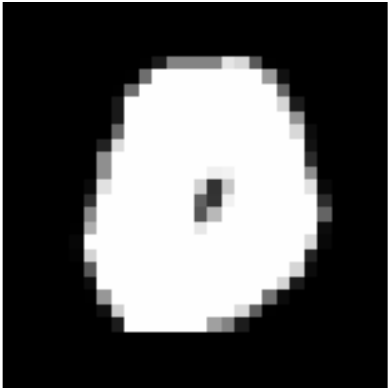}
    }
    \subfloat[layer-1 top-3]{
        \centering
        \includegraphics[width=.2\linewidth]{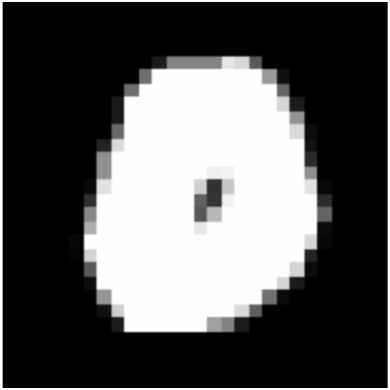}
    }
\\
    \subfloat[layer-2 top-1]{
        \centering
        \includegraphics[width=.2\linewidth]{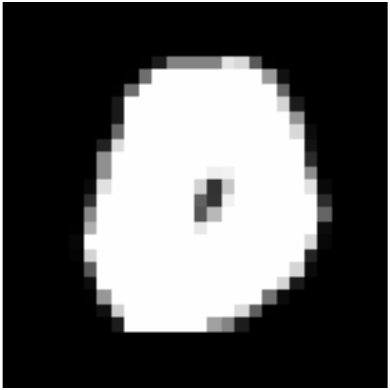}
    }
    \subfloat[layer-2 top-2]{
        \centering
        \includegraphics[width=.2\linewidth]{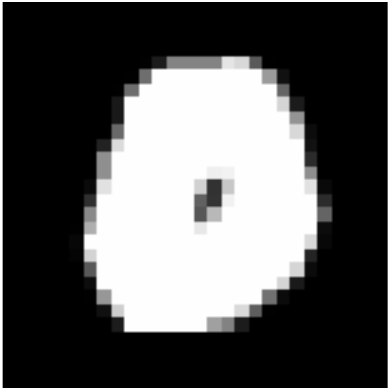}
    }
    \subfloat[layer-2 top-3]{
        \centering
        \includegraphics[width=.2\linewidth]{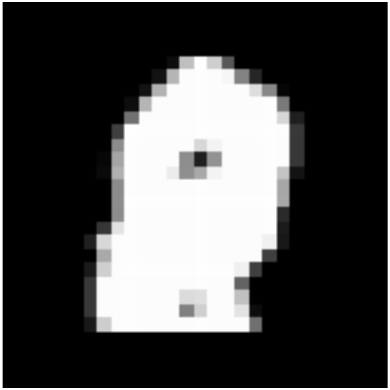}
    }
    \caption{Top scoring examples for the misclassified MNIST input of Figure \ref{sfig:app_input_seq_misclassified} in the \textbf{continual training case}.}
    \label{app:fig:misclassified_mnist}
\end{figure*}

\begin{figure*}[ht]
\centering
    \subfloat[layer-0 top-1]{
        \centering
        \includegraphics[width=.2\linewidth]{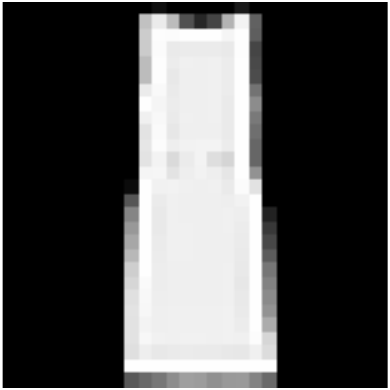}
    }
    \subfloat[layer-0 top-2]{
        \centering
        \includegraphics[width=.2\linewidth]{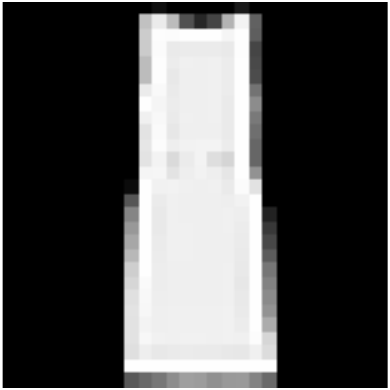}
    }
    \subfloat[layer-0 top-3]{
        \centering
        \includegraphics[width=.2\linewidth]{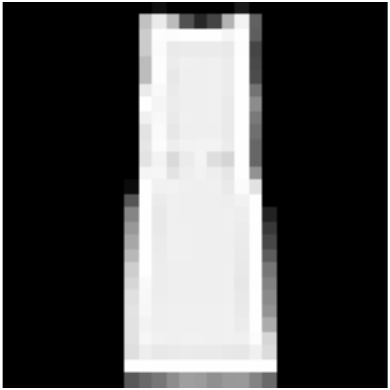}
    }
\\
    \subfloat[layer-1 top-1]{
        \centering
        \includegraphics[width=.2\linewidth]{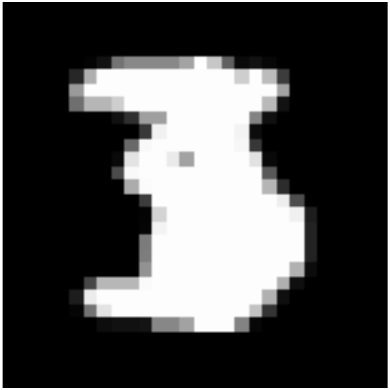}
    }
    \subfloat[layer-1 top-2]{
        \centering
        \includegraphics[width=.2\linewidth]{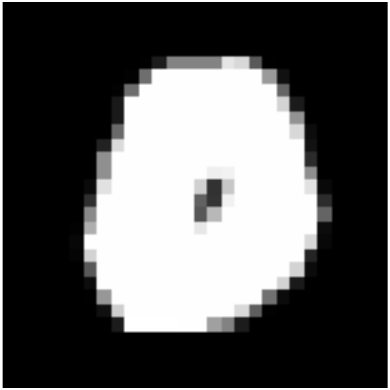}
    }
    \subfloat[layer-1 top-3]{
        \centering
        \includegraphics[width=.2\linewidth]{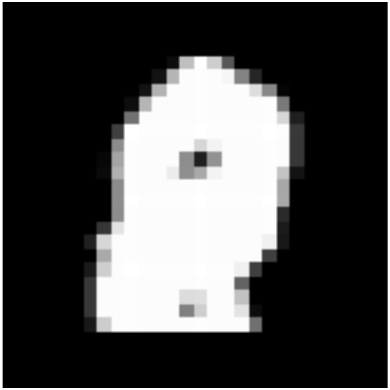}
    }
\\
    \subfloat[layer-2 top-1]{
        \centering
        \includegraphics[width=.2\linewidth]{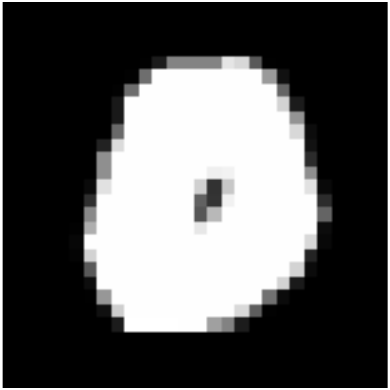}
    }
    \subfloat[layer-2 top-2]{
        \centering
        \includegraphics[width=.2\linewidth]{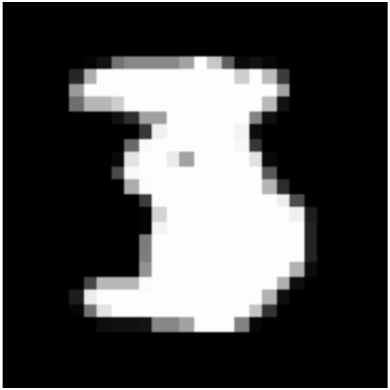}
    }
    \subfloat[layer-2 top-3]{
        \centering
        \includegraphics[width=.2\linewidth]{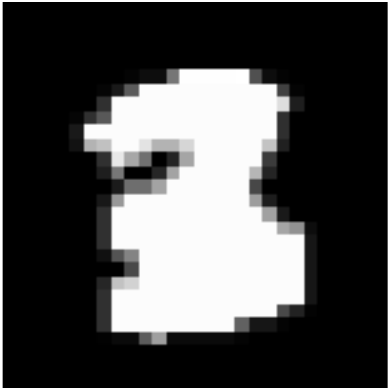}
    }
    \caption{Top scoring examples for the correctly classified F-MNIST input of Figure \ref{sfig:app_input_seq_fmnist} in the \textbf{continual training case}.}
\label{app:fig:seq_f_mnist}
%    \caption{Index 5 Top examples for sequential\_correct\_1}
\end{figure*}

\begin{figure*}[ht]
	\subfloat[layer-0]{
		\centering
		\includegraphics[width=.32\linewidth]{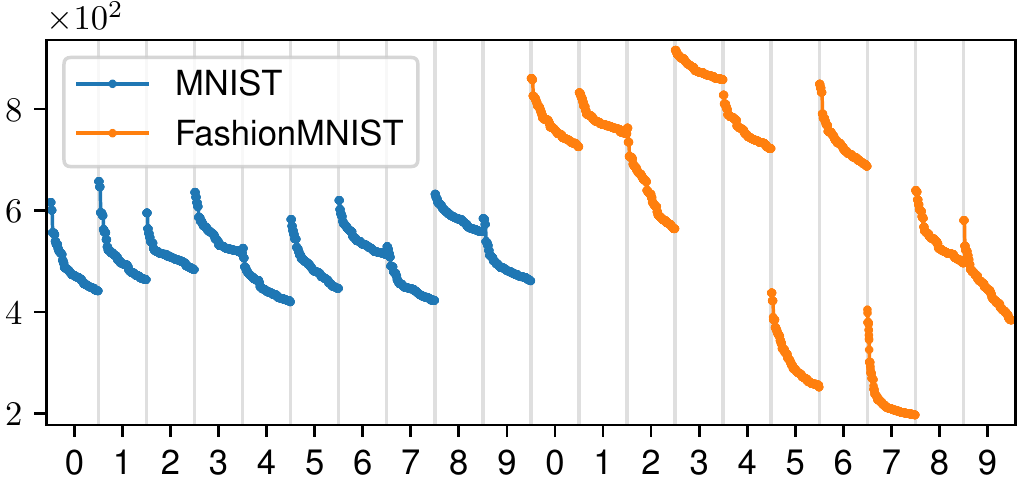}
%		\label{sfig:f_mnist_layer0}
	}
	\subfloat[layer-1] {
		\centering
		\includegraphics[width=.32\linewidth]{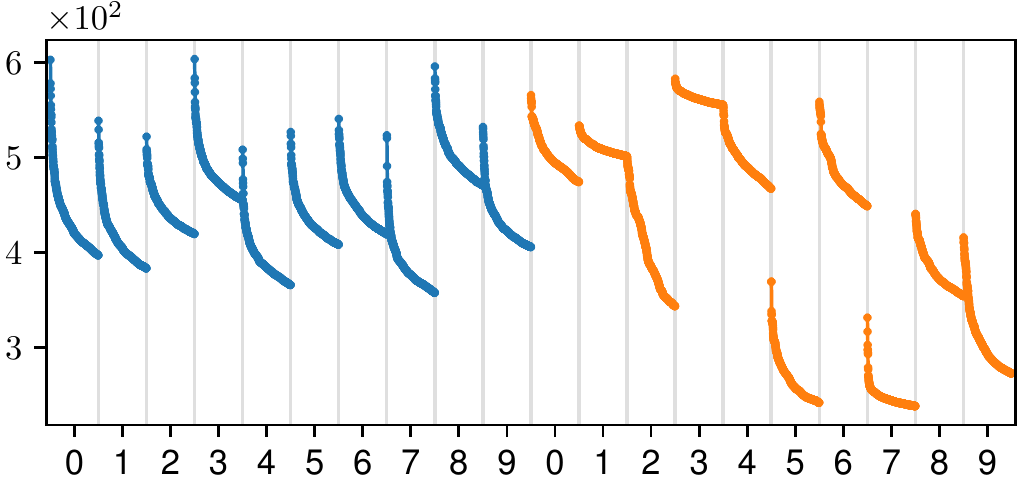}
%		\label{sfig:f_mnist_layer1}
	}
	\subfloat[layer-2]{
		\centering
		\includegraphics[width=.32\linewidth]{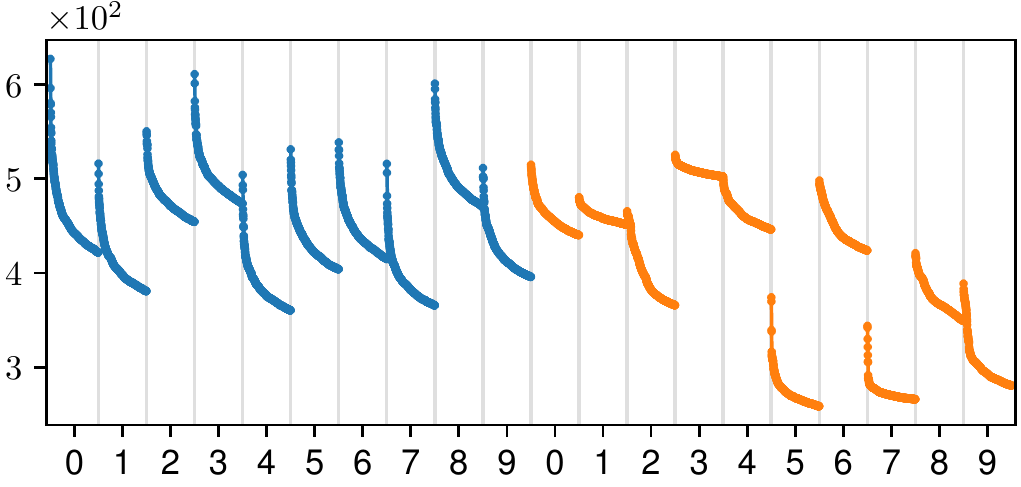}
%		\label{sfig:f_mnist_layer2}
	}
	%\vspace{-3mm}
	% Index 5 
	\caption{\it Attention weights over training examples for the input test example from class 3 of F-MNIST (Figure \ref{sfig:app_input_seq_fmnist}) in the \textbf{continual training case}. The x-axis is partitioned by class, and for each class, top-500 datapoints sorted in descending order are shown.}
 	\label{fig:app_fmnist_seq}
		\vspace{-3mm}
\end{figure*}

\begin{figure*}[ht]
	\subfloat[layer-0]{
		\centering
		\includegraphics[width=.32\linewidth]{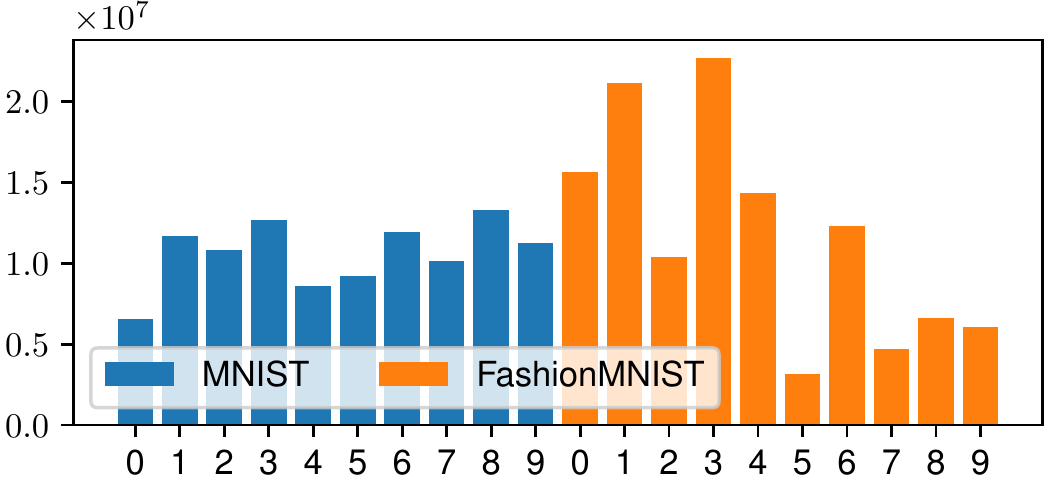}
%		\label{sfig:f_mnist_layer0_sum}
	}
	\subfloat[layer-1] {
		\centering
		\includegraphics[width=.32\linewidth]{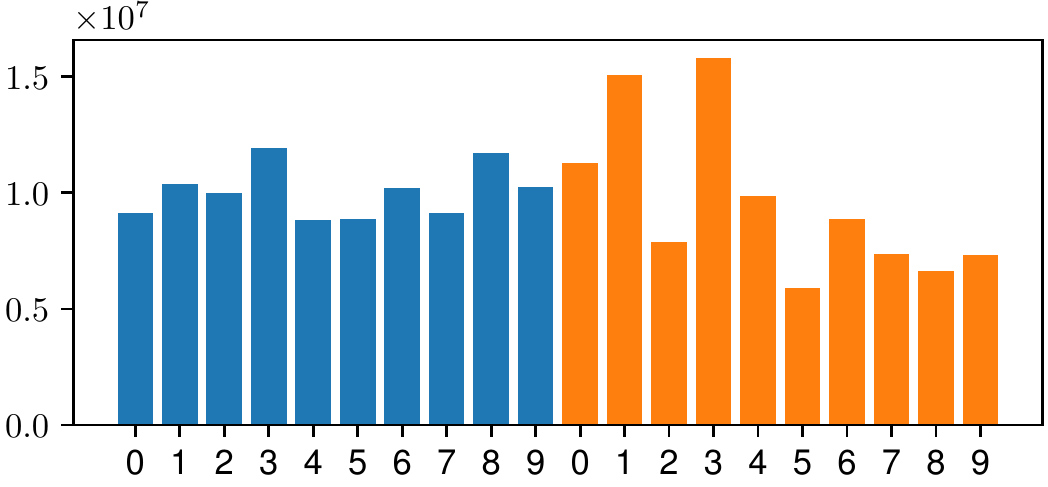}
%		\label{sfig:f_mnist_layer1_sum}
	}
	\subfloat[layer-2]{
		\centering
		\includegraphics[width=.32\linewidth]{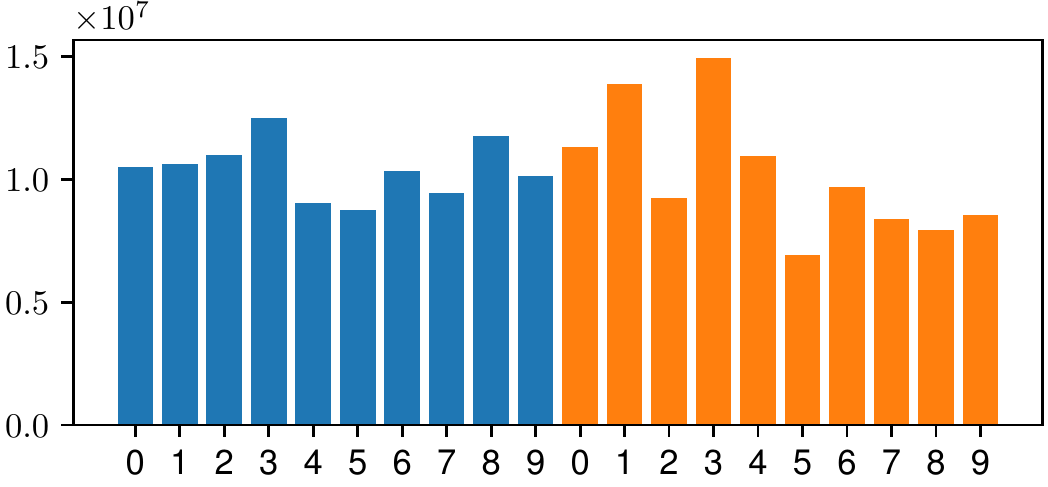}
%		\label{sfig:f_mnist_layer2_sum}
	}
	%\vspace{-3mm}
	% Index 5
	\caption{\it Total scores per class
	for the input test example from class 3 of F-MNIST (Figure \ref{sfig:app_input_seq_fmnist})
	in the joint training case.}
 	\label{fig:app_fmnist_seq_sum}
		\vspace{-3mm}
\end{figure*}

\begin{figure*}[ht]
\centering
    \subfloat[layer-0 top-1]{
        \centering
        \includegraphics[width=.2\linewidth]{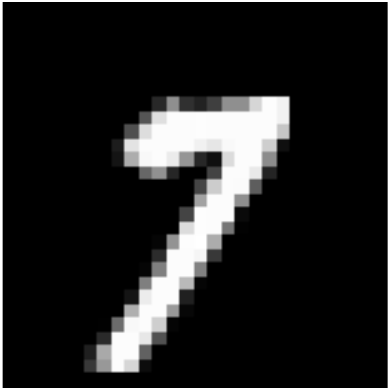}
    }
    \subfloat[layer-0 top-2]{
        \centering
        \includegraphics[width=.2\linewidth]{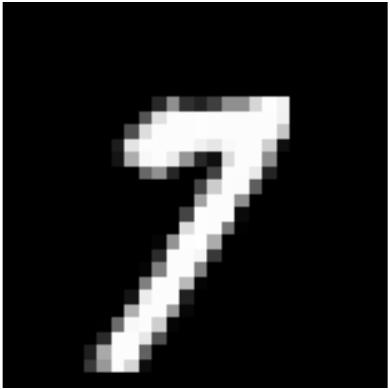}
    }
    \subfloat[layer-0 top-3]{
        \centering
        \includegraphics[width=.2\linewidth]{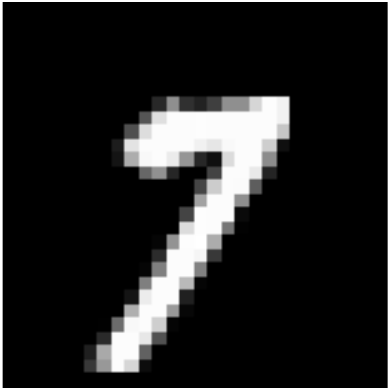}
    }
\\
    \subfloat[layer-1 top-1]{
        \centering
        \includegraphics[width=.2\linewidth]{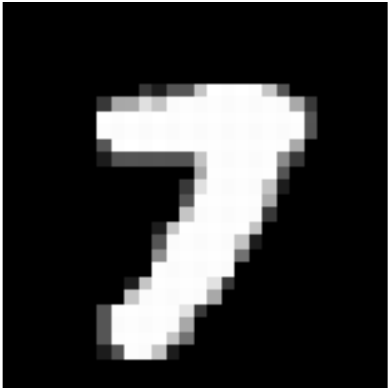}
    }
    \subfloat[layer-1 top-2]{
        \centering
        \includegraphics[width=.2\linewidth]{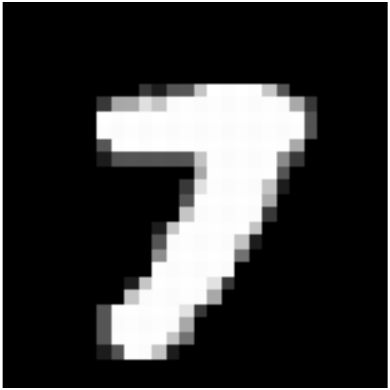}
    }
    \subfloat[layer-1 top-3]{
        \centering
        \includegraphics[width=.2\linewidth]{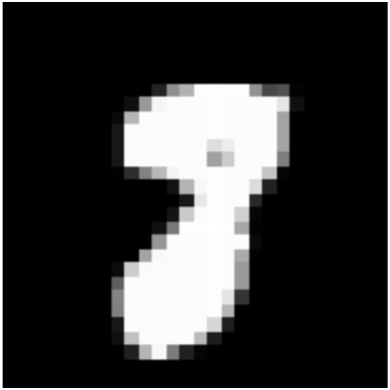}
    }
\\
    \subfloat[layer-2 top-1]{
        \centering
        \includegraphics[width=.2\linewidth]{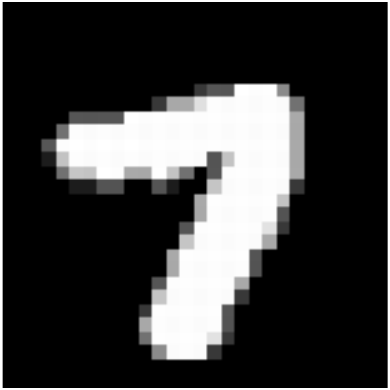}
    }
    \subfloat[layer-2 top-2]{
        \centering
        \includegraphics[width=.2\linewidth]{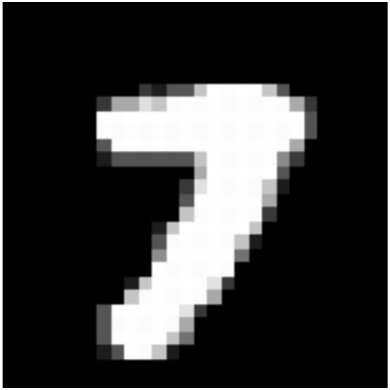}
    }
    \subfloat[layer-2 top-3]{
        \centering
        \includegraphics[width=.2\linewidth]{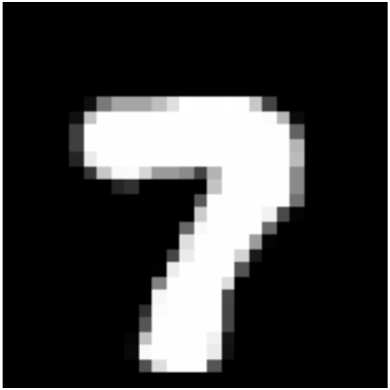}
    }
    \caption{Top scoring examples for the correctly classified MNIST input of Figure \ref{sfig:app_input_seq_mnist} in the \textbf{continual training case}.}
    \label{app:fig:seq_mnist}
%     \caption{Index 6 Top examples for sequential\_correct\_2}
\end{figure*}

\begin{figure*}[ht]
	\subfloat[layer-0]{
		\centering
		\includegraphics[width=.32\linewidth]{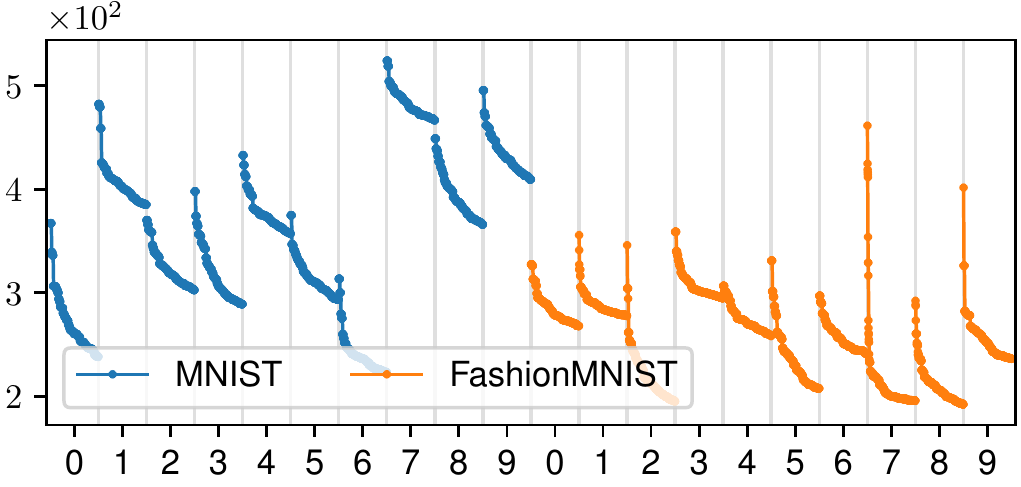}
%		\label{sfig:f_mnist_layer0}
	}
	\subfloat[layer-1] {
		\centering
		\includegraphics[width=.32\linewidth]{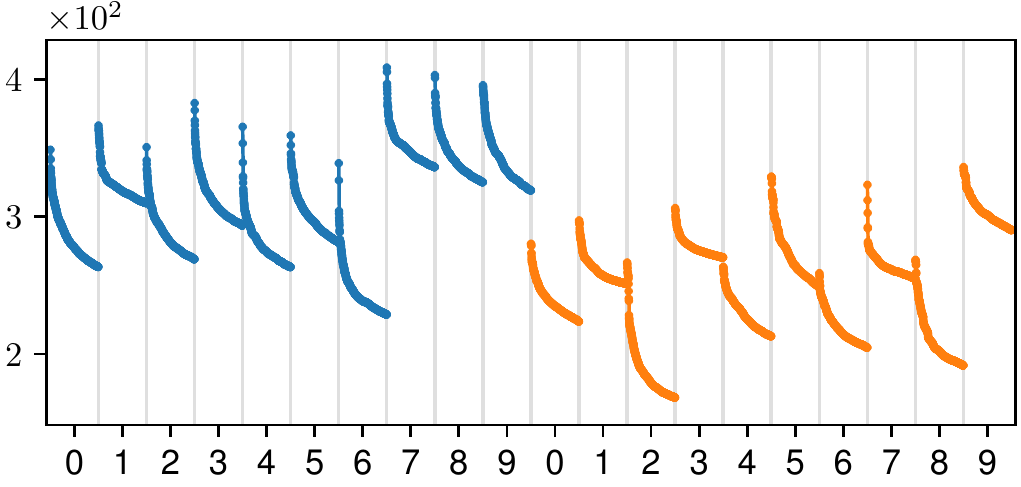}
%		\label{sfig:f_mnist_layer1}
	}
	\subfloat[layer-2]{
		\centering
		\includegraphics[width=.32\linewidth]{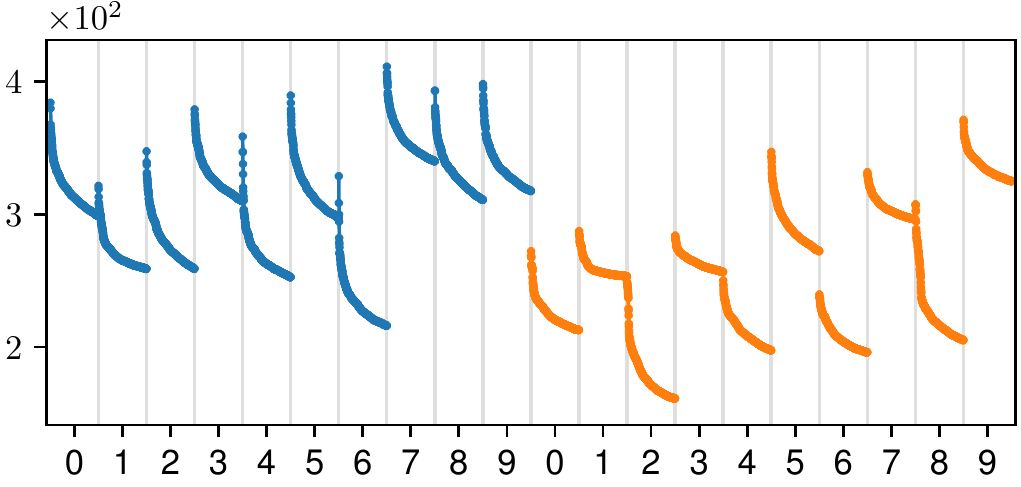}
%		\label{sfig:f_mnist_layer2}
	}
	%\vspace{-3mm}
	% Index 6 
	\caption{\it Attention weights over training examples for the input test example from class 7 of MNIST (Figure \ref{sfig:app_input_seq_mnist}) in the \textbf{continual training} case. The x-axis is partitioned by class, and for each class, top-500 datapoints sorted in descending order are shown.}
 	\label{fig:app_mnist_seq}
		\vspace{-3mm}
\end{figure*}

\begin{figure*}[ht]
	\subfloat[layer-0]{
		\centering
		\includegraphics[width=.32\linewidth]{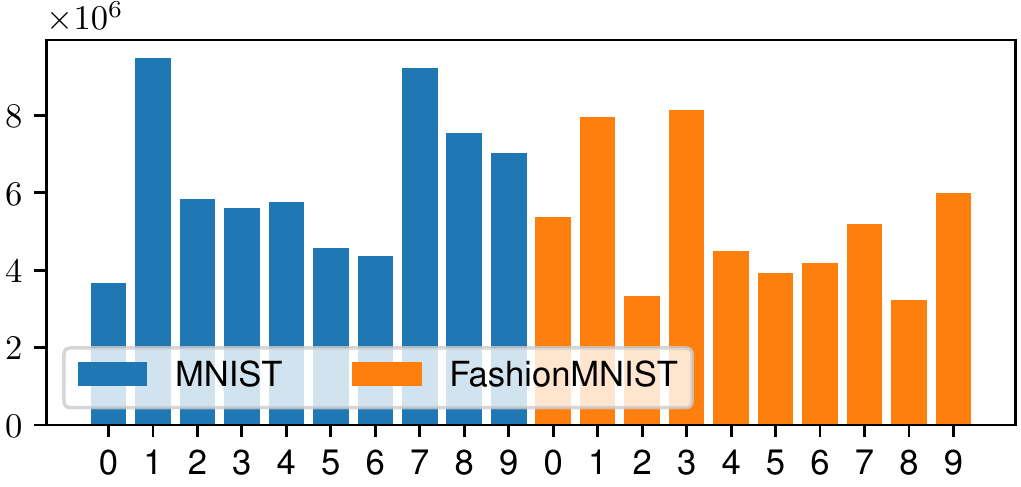}
%		\label{sfig:f_mnist_layer0_sum}
	}
	\subfloat[layer-1] {
		\centering
		\includegraphics[width=.32\linewidth]{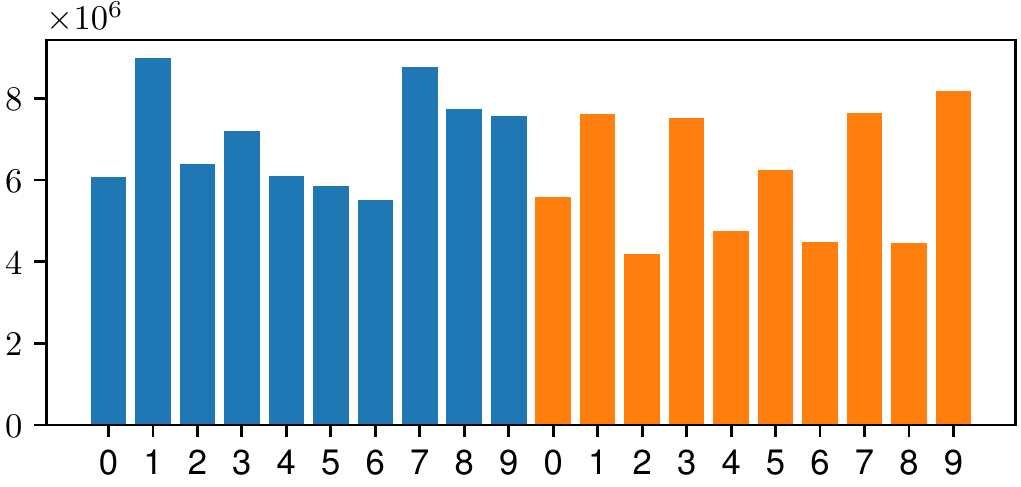}
%		\label{sfig:f_mnist_layer1_sum}
	}
	\subfloat[layer-2]{
		\centering
		\includegraphics[width=.32\linewidth]{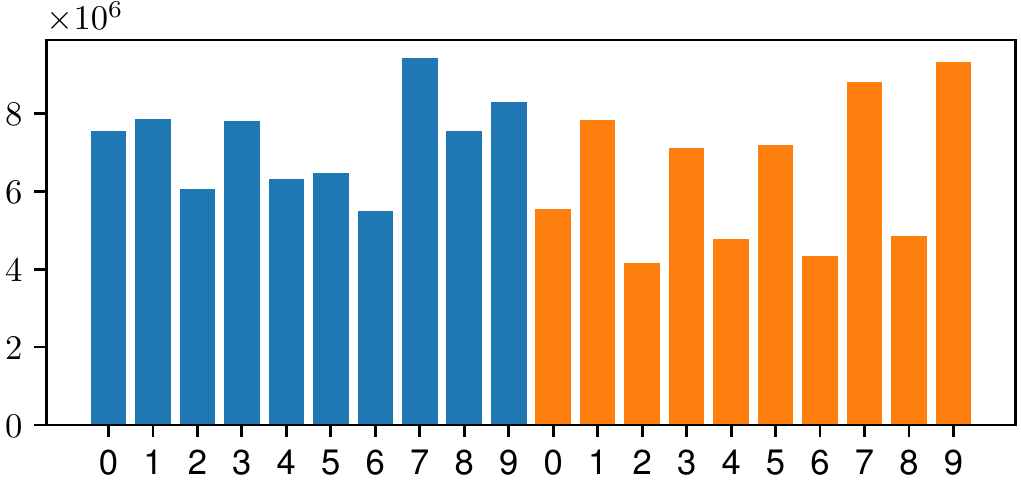}
%		\label{sfig:f_mnist_layer2_sum}
	}
	%\vspace{-3mm}
	% Index 6
	\caption{\it Total scores per class
	for the input test example from class 7 of MNIST (Figure \ref{sfig:app_input_seq_mnist})
	in the \textbf{continual training} case.}
 	\label{fig:app_mnist_seq_sum}
		\vspace{-3mm}
\end{figure*}

\section{Language Modelling Experiments}
\label{app:lm}

\subsection{Basic Settings}

Here we describe experimental details of our language modelling experiments
introduced in Sec.\ref{sec:lm}
and provide more examples.

We train a one-layer LSTM language model (LM) on two small datasets: a tiny public domain book,      ``Aesop's Fables'', by J. H. Stickney (publicly available under the project Gutenberg\footnote{\url{https://www.gutenberg.org/files/49010/49010-0.txt}}) and the standard WikiText-2 dataset \citep{merity2016pointer}.

We train a word-level LM on WikiText-2 (about 2\,M running words) and a character-level LM on the book (about 185\,K running characters).
The word-level model has a vocabulary size of 33\,K,
and the input word embedding and the LSTM dimension of 200.
For the character-level LM, the vocabulary size is 107, with an input embedding size of 64 and an LSTM layer of size 1024.
For Aesop's Fables, we isolated the last parts of the book (containing alternative versions of the tales) as a test set to generate test queries.
For WikiText-2, we use the regular train/valid/test splits.

As has been mentioned in Sec.\ref{sec:lm}, 
we focus on analysing the linear layer in the LSTM RNN layer (i.e. the single linear transformation which groups all projections including all transformations for the gates).
The input to this linear layer consists of one input coming from the previous layer and the  LSTM state from the previous time step.

Please find the corresponding analysis for the character-level LM on Aesop's Fables in
Sec.~\ref{app:lm:aesop}, and 
for the word-level LM on WikiText-2 in
Sec.~\ref{app:lm:wiki}.
Generally, we found all examples interesting
and intuitive.

\subsection{Character-Level Experiments on ``Aesop's Fables''}
\label{app:lm:aesop}

Tables \ref{tab:aesop1}, \ref{tab:aesop2} and \ref{tab:aesop3} show the queries we used and the corresponding top scoring training text passages
for Aesop's Fables.

\begin{table}[h]
\caption{Example test query and top-3 training passages from Aesop's Fables.
The test query token and the top scoring training token are highlighted in \bluet{bold}.
The query text is taken from the test set.
We see training examples with a concept of ``doing something fast'' getting high attention scores.}
\label{tab:aesop1}
% \vskip 0.15in
\begin{center}
\begin{tabular}{l|l}
\toprule
Query & ... Wolf was glad
to take himself off as f\bluet{a}st as his legs would carry him. ... \\ \midrule
Top-1 & When the Hare awoke, the Tortoise was not in sight; and running as f\bluet{a}st \\
 & as he could, he found her comfortably dozing at their goal. \\ \midrule
Top-2 & But the Stork with his long legs easily followed them to the water, \\
 & and
kept on eating them as f\bluet{a}st as he could. \\ \midrule
Top-3 & The poor Mule made room for him as f\bluet{a}st as he could, and the Horse went
proudly on his way. \\ 
\bottomrule
\end{tabular}
\end{center}
\end{table}

\begin{table}[h]
\caption{Example test query and top-3 training passages from Aesop's Fables.
The test query token and the top scoring training token are highlighted in \bluet{bold}.
The query text is taken from the test set.
We see training examples with a phrase of form ``a + adjective/single + word starting with d'' getting high attention scores.}
\label{tab:aesop2}
% \vskip 0.15in
\begin{center}
\begin{tabular}{l|l}
\toprule
Query & The
Wolf stood high up the stream and the Lamb a little \bluet{d}istance below.\\ \midrule
Top-1 & when the Pigeons had let him come in, they found that \\
& he slew more
of them in a single \bluet{d}ay than the Kite could \\ \midrule
Top-2 & all the advantages that you mention, yet when I hear the \\
& bark of but a
single \bluet{d}og, I faint with terror \\ \midrule
Top-3 & GREAT Cloud passed rapidly over a country which was \\
& parched by heat,
but did not let fall a single \bluet{d}rop to refresh it. \\ 
\bottomrule
\end{tabular}
\end{center}
\end{table}

\begin{table}[h]
\caption{Example test query and top-3 training passages from Aesop's Fables.
The test query token and the top scoring training token are highlighted in \bluet{bold}. Here the query is from the training text.
We see training examples with a phrase of form ``at'' plus some timing or counting related concept getting high attention scores.
}
\label{tab:aesop3}
% \vskip 0.15in
\begin{center}
\begin{tabular}{l|l}
\toprule
Query & The Squirrel takes a look at them—he can do no more. At \bluet{SPACE} one time he is
called away; \\
& at another, even dragged off in the Lion’s service.\\ \midrule
Top-1 & ... at seeing an elephant. Is it his great bulk that you so much admire?
Mere size is nothing. \\
& At \bluet{SPACE} most it can only frighten little girls and boys \\ \midrule
Top-2 & At times he would snap at his prey, and at \bluet{SPACE} times play \\
& with him and
lick him with his tongue, ...\\
 \midrule
Top-3 & As the Log did not move, they swam round it, keeping a safe distance
away, \\
& and at \bluet{SPACE} last one by one hopped upon it.\\ 
\bottomrule
\end{tabular}
\end{center}
\end{table}

\subsection{Word-Level Experiments on WikiText-2}
\label{app:lm:wiki}
Tables \ref{tab:wiki2}, \ref{tab:wiki1} and \ref{tab:wiki3} show the examples for the word-level language model trained on WikiText-2.

\begin{table}[h]
\caption{Example test query and top-3 training passages (with their Wikipedia page title) from WikiText-2.
The test query token and the top scoring training token are highlighted in \bluet{bold}. The query is from the test text.
We see some training passages about ``some contributions of somebody on something'' getting high attention scores.}
\label{tab:wiki1}
% \vskip 0.15in
\begin{center}
\begin{tabular}{l|l}
\toprule
% https://en.wikipedia.org/wiki/Josepha_Petrick_Kemarre#:~:text=Josepha%20Petrick%20Kemarre%20is%20an,Springs%20in%20Australia's%20Northern%20Territory.
Query & her painting was printed opposite that of Tommy Watson , who was by this time famous ,\\
(\textit{Josepha Petrick Kemarre}) & particularly for his \bluet{contribution} to the design of a new building for ...\\
 \midrule
 % https://en.wikipedia.org/wiki/Laurence_Olivier
Top-1 (\textit{Laurence Olivier}) & In February 1960 , for his \bluet{contribution} to the film industry ,\\
 & Olivier was inducted into the Hollywood Walk of Fame  \\
\midrule
% https://en.wikipedia.org/wiki/Khoo_Kheng-Hor
Top-2 (\textit{Khoo Kheng-Hor}) & he was appointed as honorary Assistant Superintendent of Police by\\
 & the Singapore Police Force in recognition for his \bluet{contribution} as consultant \\
 \midrule
 % https://en.wikipedia.org/wiki/History_of_artificial_intelligence
Top-3 (\textit{History of AI}) & Colby did not credit Weizenbaum for his \bluet{contribution} to the program .\\
\bottomrule
\end{tabular}
\end{center}
\end{table}

\section{Further Discussion on Scalability}
\label{app:scalability}
We further discuss the scalability of the analysis introduced here for larger models.
The main requirement for the analysis presented here is to store training datapoints during training, whose size linearly increases with the number of training steps.
The complexity of test-time attention weight computation is also linear w.r.t.~the number of training steps, and it is just a one-time computation, which is nothing compared to the resources needed for training the model.
The largest GPT3 model has a state size of 12K and is trained on 300B tokens.
The dual form of one self-attention layer would thus require 300G * 12K * 4 = 14PB storage.
This is huge, but not infeasible.
Alternatively, if the training is reproducible, we could opt for not storing training datapoints: train the model once, compute test queries, then re-train the model to recompute the training datapoints to compute
test attention weights and store only the statistics relevant for the analysis (e.g., top-k and sum).

\begin{table}[h]
\caption{Example test query and top-3 training passages (with their Wikipedia page title) from WikiText-2.
The test query token and the top scoring training token are highlighted in \bluet{bold}. The query is from the test text.
We see training examples about ``somebody rating (or commenting on) something (e.g.~movie/song)''
getting high attention scores.
}
\label{tab:wiki3}
% \vskip 0.15in
\begin{center}
\begin{tabular}{l|l}
\toprule
% https://en.wikipedia.org/wiki/The_Snowmen
Query & ... IGN 's Matt gave " The " a score of 9 @.@ 4 out of 10 , describing \bluet{it} as " a ... \\
(\textit{The Snowmen}) &  in storytelling " which " refreshingly " lacked traditional Christmas references \\
 \midrule
 % https://en.wikipedia.org/wiki/Irresistible_(2020_film)
Top-1  & Den of Geek writer named it the " finest " stand @-@ alone episode of the second season ,\\
(\textit{Irresistible (film)}) & describing \bluet{it} as " a genuinely creepy 45 @-@ minute horror movie " ... \\
\midrule
% https://en.wikipedia.org/wiki/Moment_of_Surrender
Top-2 & NME felt that it was the " most impressive " song on the album , describing \bluet{it} as a \\
(\textit{Moment of Surrender}) & " gorgeously sparse prayer built around Adam Clayton 's bassline and Bono 's rough " ... \\
 \midrule
 % https://en.wikipedia.org/wiki/Species_(film)
Top-3 (\textit{Species (film)}) & James from magazine gave the film 2 out of 5 stars , describing \bluet{it} as " ' Alien ' meets ... \\
\bottomrule
\end{tabular}
\end{center}
\end{table}

\end{document}